\documentclass[a4paper]{article}

\usepackage{amsthm}
\usepackage{amssymb}
\usepackage{amsmath}
\usepackage{stmaryrd}

\usepackage{tikz}
\usetikzlibrary{arrows,automata}

\usepackage{hyperref}
\hypersetup{bookmarks = true, colorlinks = false}

\usepackage
{geometry}

\usepackage{color}
\usepackage{gb4e}
\noautomath

\theoremstyle{plain}
\newtheorem{definition}{Definition}

\theoremstyle{plain}
\newtheorem{lemma}{Lemma}

\theoremstyle{plain}
\newtheorem{theorem}{Theorem}

\theoremstyle{plain}
\newtheorem{example}{Example}

\theoremstyle{plain}
\newtheorem{fact}{Fact}

\newcommand{\PL}{\mathsf{PL}}
\newcommand{\conwon}{\mathsf{ConWON}}

\newcommand{\AP}{\mathsf{AP}}

\newcommand{\AAA}{\mathbf{A}}
\newcommand{\EEE}{\mathbf{E}}

\newcommand{\cond}[1]{[#1]}

\newcommand{\dcon}[1]{\langle #1 \rangle}

\newcommand{\CC}{\mathtt{C}}
\newcommand{\DD}{D}

\newcommand{\SC}{\mathtt{C}}

\newcommand{\DDD}{\mathbb{D}}

\newcommand{\sete}[1]{\lVert #1 \rVert}

\newcommand{\updd}[2]{#1 \oplus #2}

\newcommand{\HI}[1]{\mathbf{HI}(#1)}

\newcommand{\EN}[1]{\bigcap #1}

\newcommand{\MM}{\mathtt{M}}

\newcommand{\defg}[1]{|#1|}

\newcommand{\updf}[2]{#1 + #2}

\newcommand{\logv}{\mathsf{V}}

\newcommand{\cona}{\rhd}

\newcommand{\stas}[1]{|#1|}

\newcommand{\exas}[1]{\textit{#1}}

\newcommand{\ecs}[1]{{\color{blue}(#1)}}

\newcommand{\fcut}[1]{}

\newcommand{\bee}{\begin{exe}}
\newcommand{\eee}{\end{exe}}

\newcommand{\defstyle}{\textbf}

\title{A logical theory for conditional weak ontic necessity \\ based on context update\footnote{During the development of this paper, I had many valuable meetings with and got many useful comments from Valentin Goranko. I would like to thank him for all the kind help. Thanks also go to the audience of seminars at Beijing Normal University and Southwest University and a conference at Nankai University. This research was supported by the National Social Science Foundation of China (No. 19BZX137).}}
\author{
Fengkui Ju \smallskip \\
{\small School of Philosophy, Beijing Normal University, Beijing, China} \\
{\small \emph{fengkui.ju@bnu.edu.cn}}
}
\date{}

\begin{document}

\maketitle

\begin{abstract}
\noindent Weak ontic necessity is the ontic necessity expressed by ``should'' or ``ought to'' in English. An example of it is ``I should be dead by now''. A feature of this necessity is whether it holds does not have anything to do with whether its prejacent holds. In this paper, we present a logical theory for conditional weak ontic necessity based on context update. A context is a set of ordered defaults, determining expected possible states of the present world. Sentences are evaluated with respect to contexts. When evaluating the conditional weak ontic necessity with respect to a context, we first update the context with the antecedent, then check whether the consequent holds with respect to the updated context. The logic is complete. Our theory combines premise semantics and update semantics for conditionals.

\medskip

\noindent \textbf{Keywords:} Conditional weak ontic necessity; Ordered defaults; Expected possible states; Update
\end{abstract}

%%%%%%%%%%%%%%%%%%%%%%%%%%%%%%%%%%%%%%%%
\section{Introduction}
\label{sec:Introduction}
%%%%%%%%%%%%%%%%%%%%%%%%%%%%%%%%%%%%%%%%

%%%%%%%%%%%%%%%%%%%%%%%%%%%%%%%%%%%%%%%%
\subsection{Modalities}
\label{subsec:Modalities}
%%%%%%%%%%%%%%%%%%%%%%%%%%%%%%%%%%%%%%%%

%\subsubsection{}

\emph{Modalities} locate their \emph{prejacents} in spaces of possibilities~\cite{vonFintel06}. Different modalities locate their prejacents in different kinds of spaces: \emph{epistemic}, \emph{deontic}, and so on, that is, different modalities have different \emph{flavors}~\cite{Matthewson16}. Different modalities locate their prejacents in different ways: \emph{universal}, \emph{existential}, and so on, that is, different modalities have different \emph{forces}~\cite{Matthewson16}. The modalities related to universal quantifications are called \emph{necessities} and those related to existential quantifications are called \emph{possibilities}\footnote{Note that there are two different senses for the word \emph{possibility} in this paragraph.}. We look at two examples:

\begin{exe}
\ex \exas{The man talking aloud \textbf{might} be drunk.}
\ex \exas{Bob \textbf{must} go to school.}
\end{exe}

\noindent The first sentence says that ``\exas{The man talking aloud is drunk}'' is true at \emph{some epistemic} possibility, and the second one says that ``\exas{Bob goes to school}'' is true at \emph{all deontic} possibilities.

%\subsubsection{}

Modalities are complex and there are many theories of them in the literature. We refer to \cite{vonFintel06} and \cite{Matthewson16} for some general discussions. The most influential theory in linguistics was proposed by Kratzer in \cite{Kratzer91}, among her other works. Its framework consists of a set of possible worlds and two functions, respectively called a \emph{modal base} and an \emph{ordering source}. For every possible world, a modal base specifies a set of propositions, determining the accessible worlds to this world, and an ordering source also specifies a set of propositions, ordering possible worlds from the perspective of this world. Many modalities can be interpreted in this framework.

%%%%%%%%%%%%%%%%%%%%%%%%%%%%%%%%%%%%%%%%
\subsection{Conditional modalities}
\label{subsec:Conditional modalities}
%%%%%%%%%%%%%%%%%%%%%%%%%%%%%%%%%%%%%%%%

%\subsubsection{}

Conditional modalities are those sentences claiming that a modalized sentence is the case in a proposed scenario, which may or may not be actual~\cite{vonFintel2011}. Here are two examples:

\bee
\ex \exas{If the baby crying is not hungry, he \textbf{must} be angry.} \label{exe:If the baby crying is not hungry}
\ex \exas{If I were there now, I \textbf{should} help the victims as well.} \label{exe:If I were there now}
\eee

What are called \emph{conditionals} are always conditional modalities and there are no genuinely bare conditionals~\cite{Kratzer86}. Those seemly bare conditionals have implicit modalities. For example, the following sentence has an implicit epistemic necessity:

\bee
\ex \exas{If the man approaching is not Jack, he is Zack.}
\eee

It is not always clear what the implicit modality is in seemly bare conditionals. Here is an example.

\bee
\ex \exas{If the match were struck, it would light.} \label{exe:If the match were struck}
\eee

\noindent Leitgeb \cite{Leitgeb12} claims that the implicit modality in this sentence is ``necessary'' but Wawer and Wro\'{n}ski \cite{WawerWronski15} disagree.

Conditional modalities can be classified into two classes from a semantic perspective: those whose antecedent is compatible with the speaker's beliefs and those whose antecedent is contrary to the speaker's beliefs. The latter is commonly called \emph{counterfactual} conditionals. No name for the former has been widely established in the literature yet. Following some sources, such as \cite{Rott99}, we call them \emph{open} conditionals. For example, Sentence \ref{exe:If the baby crying is not hungry} is an open conditional and Sentence \ref{exe:If I were there now} is a counterfactual conditional in typical situations.

There is a syntactic distinction between English conditionals: \emph{indicative} and \emph{subjunctive} ones. For example, Sentence \ref{exe:If the baby crying is not hungry} is an indicative conditional and Sentence \ref{exe:If I were there now} is a subjunctive one. The relation between this syntactic distinction and the abovementioned semantic distinction is complicated. Simply speaking, they are not identical. We refer to \cite{Anderson51} and \cite{vonFintel2011} for some counter-examples.
 
Here is a fact. Many other languages (if not all) do not have a syntactic characterization for open and counterfactual conditionals~\cite{Yong17}. For example, there are many (if not most) conditionals in Chinese and Vietnamese, which are open or counterfactual in a situation can only be decided semantically.
 
%\subsubsection{}

How do we formally make sense of conditionals? There have been many studies on that. Generally speaking, the main research lines include \emph{ordering semantics}, \emph{premise semantics}, \emph{probability semantics}, and \emph{belief revision approach}. All these lines are greatly influenced by Ramsey's Test, proposed in \cite{Ramsey29}, and there are all kinds of connections among them. We refer to \cite{EgreRott21} for general discussions.

Here we briefly mention some main references in ordering semantics and premise semantics closely related to our work. We will discuss some of them in detail later.

Ordering semantics was proposed by Stalnaker \cite{Stalnaker68} and Lewis \cite{Lewis73}, which is the most commonly accepted theory for counterfactual conditionals. The core idea of this theory is that a counterfactual is true in the present state of the world if its consequent is true in all the alternative states satisfying the antecedent, which are most similar to the present state. The ordering of \emph{being more similar} plays a crucial role in this theory, which is how it gets its name.

Premise semantics was proposed by Veltman \cite{Veltman76} and Kratzer \cite{Kratzer79}. Its core idea is that evaluation of a conditional in a situation involves evaluation of its consequent in situations where not only its antecedent but also some \emph{additional} premises are satisfied. This idea was already discussed in \cite{Chisholm46} and \cite{Goodman47}, among others.

The two approaches influenced many works. Here we mention two important ones. As said above, Kratzer \cite{Kratzer91} proposed a general framework to deal with modalities. This work also contains a way to handle conditionals, which combines the two approaches: Premises are used to induce orderings in this work. Veltman \cite{Veltman05} provided a semantics for counterfactual conditionals, based on an update mechanism and premise semantics.

%%%%%%%%%%%%%%%%%%%%%%%%%%%%%%%%%%%%%%%%
\subsection{Weak ontic necessity}
\label{subsec:Weak ontic necessity}
%%%%%%%%%%%%%%%%%%%%%%%%%%%%%%%%%%%%%%%%

\emph{Ontic possibilities} are possible states in which our world could be in the \emph{ontic} sense: That they are possible is independent of our knowledge about the present state of the world. Ontic possibilities can be easily confused with \emph{epistemic} ones. Epistemic possibilities are possible states in which our world could be in the \emph{epistemic} sense: That they are possible is due to our ignorance about the present state of the world.

The weak ontic necessity, explicitly identified and called a metaphysical necessity by Copley \cite{Copley06}, is the ontic necessity that can be expressed by ``should'' or ``ought to'' in English. Here are some examples:

\begin{exe}
\ex Suppose Jones is in a crowded office building when a severe earthquake hits. The building topples. By sheer accident, nothing falls upon Jones; the building just happens to crumble in such a way as not to touch the place where he is standing. He emerges from the rubble as the only survivor. Talking to the media, Jones says the following: \par \vspace{3pt} \exas{I \textbf{ought to} be dead right now.} \label{I ought to be dead right now} \ecs{From \cite{Yalcin16}}
\ex \exas{Our guests \textbf{ought to} be home by now. They left half-an-hour ago, have a fast car, and live only a few miles away. However, they are not home yet.} \ecs{Adapted from \cite{Leech71}}
\ex \exas{The beer \textbf{should} be cold by now. I put it into the refrigerator one hour ago. But I have absolutely no idea whether it is.} \ecs{Adapted from \cite{Copley06}}
\ex Consider Rasputin. He was hard to kill. First his assassins poisoned him, then they shot him, then they finally drowned him. Let us imagine that we were there. Let us suppose that the assassins fed him pastries dosed with a powerful, fast-acting poison, and then left him alone for a while, telling him they would be back in half an hour. Half an hour later, one of the assassins said to the others, confidently, ``\exas{He \textbf{ought to} be dead by now.}'' The others agreed, and they went to look. Rasputin opened his eyes and glared at them. ``\exas{He \textbf{ought to} be dead by now!}'' they said, astonished. \ecs{From \cite{Thomson08}}
\end{exe}

\noindent It can be seen that the weak necessity has nothing to do with whether its prejacent is true.

After Copley \cite{Copley06}, the weak ontic necessity was discussed by Swanson \cite{Swanson08}, Thomson \cite{Thomson08}, Finlay \cite{Finlay09}, Yalcin \cite{Yalcin16} and Ju \cite{Ju22}. Yalcin \cite{Yalcin16} gave a formal theory for the weak ontic necessity influenced by \cite{Veltman96}. Ju \cite{Ju22} presented a formal theory for the weak ontic necessity adopting a temporal perspective. Generally speaking, this modality has not received enough attention yet.

%%%%%%%%%%%%%%%%%%%%%%%%%%%%%%%%%%%%%%%%
\subsection{Conditional weak ontic necessity}
\label{subsec:Conditional weak ontic necessity}
%%%%%%%%%%%%%%%%%%%%%%%%%%%%%%%%%%%%%%%%

Conditional weak ontic necessity is common in reality. Here are some examples:

\begin{exe}
\ex In an early morning, a group of soldiers come to a villa to arrest a man. They enter this man's bedroom but do not find him. The leader says the following: \par \vspace{3pt} \exas{If his quilt is still warm, he \textbf{should} not be far yet}.
\ex \exas{My car \textbf{should} be parked on the street outside, but if it was stolen last night, it \textbf{should} be in a chop shop by now.} \ecs{Adapted from \cite{Yalcin16}} \label{exe:My car should be parked on the street outside}
\ex \exas{If her mother had taken a metro, she \textbf{should} be home by now.}
\ex \exas{If the alarm had sounded yesterday, I \textbf{should} have ignored it.} \ecs{From \cite{Dudman84}}
\end{exe}

How do we formally deal with conditional weak ontic necessity? This is not a trivial question. To see this, note that conditional weak ontic necessity is clearly not monotonic, as shown by Sentence \ref{exe:My car should be parked on the street outside}.

As far as we read, no work in the literature explicitly handles conditional weak ontic necessity yet. As said, there is not even enough work on the weak ontic necessity yet.

%%%%%%%%%%%%%%%%%%%%%%%%%%%%%%%%%%%%%%%%
\subsection{Our work}
\label{subsec:Our work}
%%%%%%%%%%%%%%%%%%%%%%%%%%%%%%%%%%%%%%%%

In this work, we present a logical theory for conditional weak ontic necessity. This theory combines three approaches. Firstly, following premise semantics, we think that additional propositions play a role in evaluating conditionals. Secondly, we think that when evaluating a conditional, we should first update something with its antecedent and then evaluate its consequent with respect to the update result, which follows the update semantics proposed by Veltman \cite{Veltman05}. Thirdly, our understanding of the weak ontic necessity follows Ju \cite{Ju22}.

Intuitively, we understand the conditional weak ontic necessity in the following way. The agent has a system of defaults, called a context, that determines which possible states of the present world are expected. The weak ontic necessity quantifies over the set of expected possible states. A conditional weak ontic necessity is true with respect to a context if its consequent is true with respect to the result of updating the context with its antecedent.

Conceptually, our logic differents from ordering semantics in dealing with conditionals. Technically, its flat fragment is equivalent to the flat fragment of the conditional logic $\logv$ proposed by Lewis \cite{Lewis73}, which is based on ordering semantics.

The rest of the paper is structured as follows. In Section \ref{sec:Our approach to conditional weak ontic necessity}, we state our approach to conditional weak ontic necessity in detail. In Section \ref{sec:A logical theory for conditional weak ontic necessity}, we present a logical theory for conditional weak ontic necessity. In Section \ref{sec:Comparisons}, we compare our theory to some closely related theories on general conditionals. In Section \ref{sec:Looking backward and forward}, we summarize our work and mention some future directions. Section \ref{sec:Proofs} contains proofs for some results.

%%%%%%%%%%%%%%%%%%%%%%%%%%%%%%%%%%%%%%%%
\section{Our approach to conditional weak ontic necessity}
\label{sec:Our approach to conditional weak ontic necessity}
%%%%%%%%%%%%%%%%%%%%%%%%%%%%%%%%%%%%%%%%

%\subsubsection{}

Fix an agent and a moment. The world can evolve in different ways from this moment and there are different possible futures. However, not all of the possible futures are \emph{expected} for the agent. For example, if it usually takes one hour by bus from your home to the airport, then those possible futures where it takes two hours are not expected.

The agent has some \emph{defaults} concerning which possible futures are expected. These defaults can be of many kinds: natural laws (\emph{Light is faster than sound}), common natural phenomena (\emph{It is cold during the winter}), planned events (\emph{Jack will go to a dinner with Jane tomorrow}), or simply customs (\emph{Jones wears a pink cap on sunny days}).

The defaults may conflict with each other and there is an \emph{order} among them. The order results from many factors. Here we mention two important ones. First, some kinds of defaults usually have higher priority than others. For example, natural laws usually have the highest priority. Second, special defaults usually have higher priority than general ones. Here is an example. Suppose there are two defaults: (1) \emph{It is cold in Beijing in winter}; (2) \emph{The El Ni\~{n}o condition causes warm winter in Beijing}. Suppose it is fall and the El Ni\~{n}o condition has occurred. In this case, we would expect a warm winter in Beijing.

The system of defaults determines which possible futures are expected and which are not.

The sentence ``\exas{If $\phi$ is the case in the future, then $\psi$ should be the case in the future}'' is true if and only if $\psi$ is true with respect to all expected possible futures determined by the result of putting $\phi$ to the system of defaults with the highest priority.

Suppose the world evolves and one possible future is realized, which might not be expected. Consequently, we have the following.

The world is in a state now, which has many alternatives. The agent has a system of defaults concerning what possible states are expected. According to it, the present state might not be expected.

The sentence ``\exas{If $\phi$ is the case now, then $\psi$ should be the case now}'' is true if and only if $\psi$ is true with respect to all expected possible states determined by the result of putting $\phi$ to the system of defaults with the highest priority.

%\subsubsection{}

The proposition $\phi$ can be true or not and the agent might know its truth value or not. Assume she does not know. Then ``\exas{If $\phi$ is the case now, then $\psi$ should be the case now}'' is an indicative conditional for her. Assume $\phi$ is true and she knows. Then ``\exas{If $\phi$ is the case now, then $\psi$ should be the case now}'' is an indicative conditional for her\footnote{In this case, the sentence seems odd. We do not have clear ideas about the reason. It is possible that ``should'' can be both an ontic modal and an epistemic modal and the oddness is related to the ambiguity. It is actually a controversial issue in the literature about whether ``should'' can be an epistemic modal. We refer to \cite{Copley06} and \cite{Yalcin16} for some discussions.}. Assume $\phi$ is false and she knows. Then ``\exas{If $\phi$ is the case now, then $\psi$ should be the case now}'' is a counterfactual conditional for her. Whether the agent knows the truth value of $\phi$ does not matter for whether she accepts the sentence ``\exas{If $\phi$ is the case now, then $\psi$ should be the case now}''.

%\subsubsection{}

The following example can illustrate our understanding.

%%%%%%%%%%%%%%%%%%%%%%%%%%%%%%%%%%%%%%%%
%%%%%%%%%%%%%%%%%%%%%%%%%%%%%%%%%%%%%%%%
\begin{example} \label{example:tiger}
A tribe captured many animals and put them in isolated cages in a room. Every evening they randomly open two cages, lock the room's door and leave. The next morning, they come and release the survivor. One day, three animals remain: a tiger, a dog, and a goat. The chief says:

\begin{exe}
\ex \exas{If the goat is still alive tomorrow, the dog \textbf{should} be dead.} \label{exe:If the goat is still alive tomorrow}
\end{exe}

\noindent The next morning they come to the front of the door. The chief says:

\begin{exe}
\ex \exas{If the goat is still alive now, the dog \textbf{should} be dead.} \label{exe:If the goat is still alive now}
\end{exe}

\noindent They open the door and see that the tiger killed the goat. The chief says:

\begin{exe}
\ex \exas{If the goat were still alive now, the dog \textbf{should} be dead.} \label{exe:If the goat were still alive now}
\end{exe}
\end{example}

Intuitively, the three sentences are true. How?

It seems natural to assume the following with the chief's default system. There are three defaults: \emph{Tigers kill dogs}, \emph{Tigers kill goats} and \emph{Dogs kill goats}. The second default has the highest priority, and the first and third ones are not comparable.

How is the first sentence true? Firstly, we put \emph{The goat will still be alive tomorrow} to the default system with the highest priority. According to the new default system, those possible futures where the goat is dead are not expected. For all expected possible futures, the dog will be dead. Then the dog should be dead tomorrow.

How are the second and third sentences true? Firstly, we put \emph{The goat is still alive now} to the default system with the highest priority. According to the new default system, those possible states of the present world where the tiger is dead are not expected. Then the dog should be dead now.

%%%%%%%%%%%%%%%%%%%%%%%%%%%%%%%%%%%%%%%%
\section{A logical theory for conditional weak ontic necessity}
\label{sec:A logical theory for conditional weak ontic necessity}
%%%%%%%%%%%%%%%%%%%%%%%%%%%%%%%%%%%%%%%%

%%%%%%%%%%%%%%%%%%%%%%%%%%%%%%%%%%%%%%%%
\subsection{Languages}
\label{subsec:Languages}
%%%%%%%%%%%%%%%%%%%%%%%%%%%%%%%%%%%%%%%%

\begin{definition}[The languages $\Phi_{\PL}$ and $\Phi_{\conwon}$] \label{def:The languages PL and ConWON}
Let $\AP$ be a countable set of atomic propositions. The language $\Phi_{\PL}$ of the Propositional Logic ($\PL$) is defined as follows, where $p$ ranges over $\AP$:

\[\alpha ::= p \mid \bot \mid \neg \alpha \mid (\alpha \land \alpha)\]

The language $\Phi_{\conwon}$ of the Logic for Conditional Weak Ontic Necessity ($\conwon$) is defined as follows, where $\alpha \in \Phi_{\PL}$:

\[\phi ::= p \mid \bot \mid \neg \phi \mid (\phi \land \phi) \mid \cond{\alpha} \phi\]
\end{definition}

\noindent The intuitive reading of $\cond{\alpha} \phi$ is \emph{if $\alpha$, $\phi$ should be true.}

What follow are some derivative expressions:

\begin{itemize}
\item The propositional connectives $\top, \lor, \rightarrow$ and $\leftrightarrow$ are defined as usual.
\item Define the dual $\dcon{\alpha} \phi$ of $\cond{\alpha} \phi$ as $\neg \cond{\alpha} \neg \phi$. This operator does not seem to have a natural meaning but we introduce it due to technical reasons.
\item Define $\Box \phi$ as $\cond{\top} \phi$, meaning \emph{$\phi$ should be true}.
\item Define the dual $\Diamond \phi$ of $\Box \phi$ as $\neg \Box \neg \phi$. This operator does not seem to have a natural meaning as well.
\item We define $\EEE \alpha$, where $\alpha \in \Phi_\PL$, as $\dcon{\alpha} \top$, meaning \emph{$\phi$ is possible}. Its dual $\AAA \alpha$ is defined as $\neg \EEE \neg \alpha$, meaning \emph{$\phi$ is necessary}.
\end{itemize}

We use $\Phi_{\conwon\text{-}\mathsf{1}}$ to denote the \emph{flat fragment} of $\Phi_\conwon$, which contains no nested conditionals.

%%%%%%%%%%%%%%%%%%%%%%%%%%%%%%%%%%%%%%%%
\subsection{Semantic settings}
\label{subsec:Semantic settings}
%%%%%%%%%%%%%%%%%%%%%%%%%%%%%%%%%%%%%%%%

%%%%%%%%%%%%%%%%%%%%%%%%%%%%%%%%%%%%%%%%
\subsubsection{Models}
\label{subsubsec:Models}
%%%%%%%%%%%%%%%%%%%%%%%%%%%%%%%%%%%%%%%%

\begin{definition}[Models] \label{def:Models}
A model is a tuple $\MM = (W, V)$, where $W$ is a nonempty set of states and $V: \mathsf{AP} \to \mathcal{P}(W)$ is a valuation.
\end{definition}

\noindent Intuitively, $W$ consists of all possible states of the world at an instant.

%%%%%%%%%%%%%%%%%%%%%%%%%%%%%%%%%%%%%%%%
\subsubsection{Contexts}
\label{subsubsec:Contexts}
%%%%%%%%%%%%%%%%%%%%%%%%%%%%%%%%%%%%%%%%

\begin{definition}[Contexts] \label{def:Contexts}
Let $\MM = (W, V)$ be a model. A pair $\CC = (\DDD, \succ)$ is called a \defstyle{context for $\MM$} if $\DDD$ is a finite (possibly empty) set of (possibly empty) subsets of $W$ and $\succ$ is an irreflexive and transitive relation on $\DDD$. The elements of $\DDD$ are called \defstyle{defaults}.
\end{definition}

\noindent Intuitively, $\DD_1 \succ \DD_2$ means that $\DD_1$ has higher priority than $\DD_2$.

%%%%%%%%%%%%%%%%%%%%%%%%%%%%%%%%%%%%%%%%
\subsubsection{Expected states by contexts}
\label{subsubsec:Expected states by contexts}
%%%%%%%%%%%%%%%%%%%%%%%%%%%%%%%%%%%%%%%%

%%%%%%%%%%%%%%%%%%%%%%%%%%%%%%%%%%%%%%%%%%%%%%%%%%%%
%%%%%%%%%%%%%%%%%%%%%%%%%%%%%%%%%%%%%%%%%%%%%%%%%%%%
\begin{definition}[Hierarchy of defaults in contexts] \label{def:Hierarchy of defaults in contexts}
Let $\CC = (\DDD, \succ)$ be a context for a model $\MM$. Define $\HI{\CC}$, \defstyle{the hierarchy of defaults in $\CC$}, as a sequence $(\DDD_0, \dots, \DDD_n)$ constructed as follows:

\begin{itemize}
\item Let $\DDD_0 = \{\DD \in \DDD \mid \text{$\DD$ is a maximal element of $\DDD$}\}$;
\item[\vdots]
\item If $\DDD_0 \cup \dots \cup \DDD_k \neq \DDD$, let $\DDD_{k+1} = \{\DD \in \DDD \mid \text{$\DD$ is a maximal element of $\DDD - (\DDD_0 \cup \dots \cup \DDD_k)$}\}$, or else stop.
\end{itemize}
\end{definition}

Here are some observations about $\HI{\CC}$. Firstly, $\HI{\CC}$ cannot be an empty sequence. Secondly, if $\DDD_0 = \emptyset$, then $n = 0$. Thirdly, $\DDD_0, \dots, \DDD_n$ are pairwise disjoint and their union is $\DDD$.

%%%%%%%%%%%%%%%%%%%%%%%%%%%%%%%%%%%%%%%%%%%%%%%%%%%%
%%%%%%%%%%%%%%%%%%%%%%%%%%%%%%%%%%%%%%%%%%%%%%%%%%%%
\begin{example}[Hierarchy of defaults in contexts] \label{example:Hierarchy of defaults in contexts}
~
\begin{itemize}
\item Let $\CC = (\DDD, \succ)$ be a context, where $\DDD = \emptyset$. Then $\HI{\CC} = \emptyset$. Here $\emptyset$ is not the empty sequence but the sequence with the empty set as its only element.
\item Let $\CC = (\DDD, \succ)$ be a context, where $\DDD = \{\DD_1, \DD_2, \DD_3\}$, $\DD_1 \succ \DD_2$ and $\DD_1 \succ \DD_3$. Then $\HI{\CC} = (\{\DD_1\},\{\DD_2,\DD_3\})$.
\end{itemize}
\end{example}

Similar ways of defining hierarchies with respect to an ordered set can also be found in some literature on social choice theory such as \cite{JiangEtal18}.

\begin{definition}[Expected states by contexts] \label{def:Expected states by contexts}
Let $\MM = (W,V)$ be a model, $\CC = (\DDD, \succ)$ be a context, and $\HI{\CC} = (\DDD_0, \dots, \DDD_n)$. Define the set $\sete{\CC}$ of \defstyle{expected states by $\CC$}, as follows:

\begin{itemize}
\item Suppose $\EN{\DDD_0} = \emptyset$. Then $\sete{\CC} := \EN{\DDD_0}$.
\item Suppose $\EN{\DDD_0} \neq \emptyset$. Then $\sete{\CC} := \EN{\DDD_0} \cap \dots \cap \EN{\DDD_k}$, where $(\DDD_0, \dots, \DDD_k)$ is the longest initial segment of $(\DDD_0, \dots, \DDD_n)$ such that $\EN{\DDD_0} \cap \dots \cap \EN{\DDD_k} \neq \emptyset$.
\end{itemize}
\end{definition}

\noindent Note that specially $\EN{\emptyset} = W$. This definition follows the following idea: From the top level of the hierarchy, consider as many levels as possible.

%%%%%%%%%%%%%%%%%%%%%%%%%%%%%%%%%%%%%%%%%%%%%%%%%
%%%%%%%%%%%%%%%%%%%%%%%%%%%%%%%%%%%%%%%%%%%%%%%%%
\begin{example}[Expected states by contexts] \label{example:Expected states by contexts}
Let $\MM = (W,V)$ be a model, where $W = \{w_1, w_2, w_3,$ $w_4\}$.

\begin{itemize}
\item Let $\CC = (\DDD, \succ)$ be a context such that $\HI{\CC} = \emptyset$. Then $\sete{\CC} = \EN{\emptyset} = W$.
\item Let $\CC = (\DDD, \succ)$ be a context such that $\HI{\CC} = (\{\emptyset\},\{W\})$. Then $\sete{\CC} = \EN{\{\emptyset\}} = \emptyset$.
\item Let $\CC$ be a context such that $\HI{\CC} = (\{\DD_1,\DD_2\},\{\DD_3\})$, where $\DD_1 = \{w_1, w_2\}, \DD_2 = \{w_2, w_3\}$ and $\DD_3 = \{w_3, w_4\}$. Then $\sete{\CC} = \EN{\{\DD_1, \DD_2\}} = \DD_1 \cap \DD_2 = \{w_2\}$.
\end{itemize}
\end{example}

%%%%%%%%%%%%%%%%%%%%%%%%%%%%%%%%%%%%%%%%
\subsubsection{Context update}
\label{subsubsec:Context update}
%%%%%%%%%%%%%%%%%%%%%%%%%%%%%%%%%%%%%%%%

\begin{definition}[Update of contexts] \label{def:Update of contexts}
Let $\CC = (\DDD, \succ)$ be a context for a model $\MM = (W,V)$. Let $\DD$ be a subset of $W$. We define the \defstyle{update of $\CC$ with $\DD$} as the context $\updd{\CC}{\DD} = (\DDD', \succ')$, where

\begin{itemize}
\item $\DDD' = \DDD \cup \{\DD\}$;
\item For all $\DD_1$ and $\DD_2$ in $\DDD'$, $\DD_1 \succ' \DD_2$ if and only if one of the following conditions is met:
\begin{itemize}
\item $\DD_1 \neq \DD$, $\DD_2 \neq \DD$ and $\DD_1 \succ \DD_2$;
\item $\DD_1 = \DD$ and $\DD_2 \neq \DD$.
\end{itemize}
\end{itemize}
\end{definition}

\noindent Intuitively, $\updd{\CC}{\DD}$ is got in the following way. First, we eliminate $\DD$ from $\CC$, if $\DD$ occurs in $\CC$. Second, we add $\DD$ to $\CC$ as the greatest element.

%%%%%%%%%%%%%%%%%%%%%%%%%%%%%%%%%%%%%%%%
\subsubsection{Contextualized pointed models}
\label{subsubsec:Contextualized pointed models}
%%%%%%%%%%%%%%%%%%%%%%%%%%%%%%%%%%%%%%%%

\begin{definition}[Contextualized pointed models] \label{def:Contextualized pointed models}
For every model $\MM$, context $\CC$ and state $w$, $(\MM, \CC, w)$ is called a \defstyle{contextualized pointed model}.
\end{definition}

%%%%%%%%%%%%%%%%%%%%%%%%%%%%%%%%%%%%%%%%%%%%%%%%%
%%%%%%%%%%%%%%%%%%%%%%%%%%%%%%%%%%%%%%%%%%%%%%%%%
\begin{example}[Contextualized pointed models] \label{example:A contextualized pointed model}
Figure \ref{fig:A contextualized pointed model} indicates a contextualized pointed model $(\MM, \CC, w_1)$, where

\begin{itemize}
\item $\DDD$ has three defaults: $\DD_1 = \{w_1, w_2\}, \DD_2 = \{w_1, w_2, w_3\}$ and $\DD_3 = \{w_4\}$;
\item $\DD_1 \succ \DD_2 \succ \DD_3$.
\end{itemize}

Then $\HI{\CC} = (\{\DD_1\},\{\DD_2\},\{\DD_3\})$ and $\sete{\CC} = \{w_1,w_2\}$.

The intuitive reading of $(\MM, \CC, w_1)$ is as follows. The present state is $w_1$, which has three alternatives: $w_2, w_3, w_4$. All four states evolve from an implicit past state $r$. By the context $\CC$, neither $w_3$ nor $w_4$ is expected.

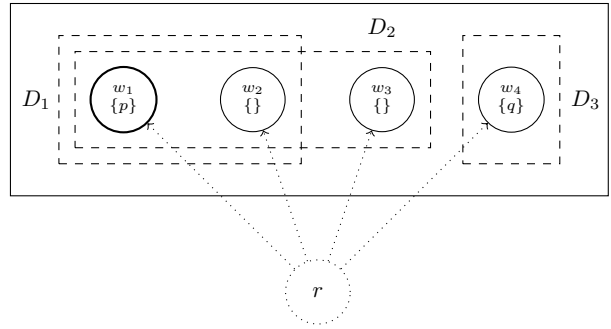
\begin{figure}[h]
\begin{center}
\begin{tikzpicture}[node distance=20mm,scale=0.85,every node/.style={transform shape}]
\tikzstyle{every state}=[draw=black,text=black,minimum size=10mm]

\draw[] (-17.5mm,-15mm) rectangle (75mm,15mm);
\draw[dashed] (-10mm,-10mm) rectangle (27.5mm,10mm);
\draw[dashed] (-7.5mm,-7.5mm) rectangle (47.5mm,7.5mm);
\draw[dashed] (52.5mm,-10mm) rectangle (67.5mm,10mm);

\node[state,thick] (s-1-1) [] {$w_1 \atop \{p\}$};
\node[state] (s-1-2) [right of=s-1-1] {$w_2 \atop \{\}$};
\node[state] (s-1-3) [right of=s-1-2] {$w_3 \atop \{\}$};
\node[state] (s-1-4) [right of=s-1-3] {$w_4 \atop \{q\}$};

\node [left=10mm] at (s-1-1) {$\DD_1$};
\node [above=8.5mm] at (s-1-3) {$\DD_2$};
\node [right=8mm] at (s-1-4) {$\DD_3$};

\node[state,dotted] (s-0-1) [below=30mm, right=5mm] at (s-1-2) {$r$};

\path
(s-0-1) edge [->,left,dotted] node {} (s-1-1)
(s-0-1) edge [->,left,dotted] node {} (s-1-2)
(s-0-1) edge [->,left,dotted] node {} (s-1-3)
(s-0-1) edge [->,left,dotted] node {} (s-1-4);

\end{tikzpicture}
\end{center}

\caption{A contextualized pointed model}
\label{fig:A contextualized pointed model}
\end{figure}
\end{example}

%%%%%%%%%%%%%%%%%%%%%%%%%%%%%%%%%%%%%%%%
\subsection{Semantics}
\label{subsec:Semantics}
%%%%%%%%%%%%%%%%%%%%%%%%%%%%%%%%%%%%%%%%

%\subsubsection{}

\begin{definition}[Semantics for $\Phi_\conwon$] \label{def:Semantics for Phi ConWON}

Consider a model $\MM$ and a context $\CC$.

\begin{itemize}
\item For every $\alpha \in \Phi_{\PL}$, we define the \defstyle{default generated by $\alpha$} as $\defg{\alpha} = \{w \mid \MM, \CC, w \Vdash \alpha\}$.
\item For every $\alpha \in \Phi_{\PL}$, we define the \defstyle{update of $\CC$ with $\alpha$} as $\updf{\CC}{\alpha} = \updd{\CC}{\defg{\alpha}}$.
\item \defstyle{Truth conditions for formulas of $\Phi_{\conwon}$ at contextualized pointed models} are defined as follows:

\medskip

\begin{tabular}{lll}
$\MM, \CC, w \Vdash p$ & $\Leftrightarrow$ & \parbox[t]{27em}{$w \in V(p)$} \\
$\MM, \CC, w \not \Vdash \bot$ & & \\
$\MM, \CC, w \Vdash \neg \phi$ & $\Leftrightarrow$ & \parbox[t]{27em}{$\MM, \CC, w \not \Vdash \phi$} \\
$\MM, \CC, w \Vdash \phi \land \psi$ & $\Leftrightarrow$ & \parbox[t]{27em}{$\MM, \CC, w \Vdash \phi$ and $\MM, \CC, w \Vdash \psi$} \\
$\MM, \CC, w \Vdash \cond{\alpha} \phi$ & $\Leftrightarrow$ & \parbox[t]{27em}{$\MM, \updf{\CC}{\alpha}, u \Vdash \phi$ for every $u \in \sete{\updf{\CC}{\alpha}}$}
\end{tabular}
\end{itemize}
\end{definition}

It can be verified that

\medskip

\begin{tabular}{lll}
$\MM, \CC, w \Vdash \dcon{\alpha} \phi$ & $\Leftrightarrow$ & \parbox[t]{27em}{$\MM, \updf{\CC}{\alpha}, u \Vdash \phi$ for some $u \in \sete{\updf{\CC}{\alpha}}$} \\
$\MM, \CC, w \Vdash \Box \phi$ & $\Leftrightarrow$ & \parbox[t]{27em}{$\MM, \CC, u \Vdash \phi$ for every $u \in \sete{\CC}$} \\
$\MM, \CC, w \Vdash \Diamond \phi$ & $\Leftrightarrow$ & \parbox[t]{27em}{$\MM, \CC, u \Vdash \phi$ for some $u \in \sete{\CC}$} \\
$\MM, \CC, w \Vdash \EEE \alpha$ & $\Leftrightarrow$ & \parbox[t]{27em}{$\MM, \CC, u \Vdash \alpha$ for some $u$} \\
$\MM, \CC, w \Vdash \AAA \alpha$ & $\Leftrightarrow$ & \parbox[t]{27em}{$\MM, \CC, u \Vdash \alpha$ for every $u$}
\end{tabular}

We say that a formula $\phi$ is \emph{valid} ($\models_\conwon \phi$) if $\MM,\CC,w \Vdash \phi$ for every contextualized pointed model $(\MM, \CC, w)$.

%\subsubsection{}

%%%%%%%%%%%%%%%%%%%%%%%%%%%%%%%%%%%%%%%%%%%%%%%%%%%%
%%%%%%%%%%%%%%%%%%%%%%%%%%%%%%%%%%%%%%%%%%%%%%%%%%%%
\begin{example} \label{example:a contextualized pointed model for the example of tiger}
We show how Example \ref{example:tiger} is analyzed in the formalization. We use $f_x$ to indicate \emph{$x$ is free} and use $a_x$ to indicate \emph{$x$ is alive}, where $x$ can be $t$ (the tiger), $d$ (the dog) or $g$ (the goat).

Figure \ref{fig:A contextualized pointed model for the example of tiger} indicates a contextualized pointed model $(\MM, \CC, w_3)$ for Example \ref{example:tiger}, where $\CC = (\DDD, \succ)$, where
\begin{itemize}
\item $\DDD$ has three defaults: $\DD_1 = \{w_1, w_3, w_4, w_5, w_6\}$, $\DD_2 = \{w_1, w_2, w_3, w_5, w_6\}$ and $\DD_3 = \{w_1, w_2, w_3, w_4, w_5\}$;
\item $\DD_2 \succ \DD_1$ and $\DD_2 \succ \DD_3$.
\end{itemize}

The default $\DD_1$ indicates that \emph{if the tiger and the dog are free, then the tiger is alive.} The defaults $\DD_2$ and $\DD_3$ are understood similarly.

The formula $\cond{a_g} \neg a_d$ means \emph{if the goat is still alive now, the dog should be dead}. It is true at $(\MM, \CC, w_3)$. How?

\begin{itemize}
\item The update of $\CC$ with $a_g$, $\updf{\CC}{a_g}$, is $(\DDD', \succ')$, where
\begin{itemize}
\item $\DDD'$ has four defaults: $\DD_1, \DD_2, \DD_3$, and $\defg{a_g} = \{w_1, w_2, w_4, w_6\}$;
\item $\defg{a_g} \succ \DD_1$, $\defg{a_g} \succ \DD_2$, $\defg{a_g} \succ \DD_3$, $\DD_2 \succ \DD_1$ and $\DD_2 \succ \DD_3$.
\end{itemize}

It can be verified that $\HI{\CC} = (\{\defg{a_g}\},\{\DD_2\},\{\DD_1, \DD_3\})$ and $\sete{\updf{\CC}{a_g}} = \{w_1\}$.
\item We can see $\MM, \updf{\CC}{a_g}, u \Vdash \neg a_d$ for every $u \in \sete{\updf{\CC}{a_g}}$. Thus, $\MM, \CC, w_3 \Vdash \cond{a_g} \neg a_d$.
\end{itemize}

%%%%%%%%%%%%%%%%%%%%%%%%%%%%%%%%%%%%%%%%%%%%%%%%%%%%
%%%%%%%%%%%%%%%%%%%%%%%%%%%%%%%%%%%%%%%%%%%%%%%%%%%%
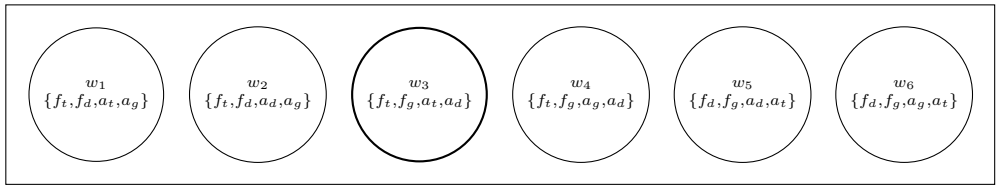
\begin{figure}[h]
\begin{center}
\begin{tikzpicture}[node distance=25mm,scale=0.85,every node/.style={transform shape}]
\tikzstyle{every state}=[draw=black,text=black,minimum size=18mm]

\draw[] (-14mm,-14mm) rectangle (139mm,14mm);

\node[state] (s-1-1) [] {$w_1 \atop \{f_t,f_d,a_t,a_g\}$};
\node[state] (s-1-2) [right of=s-1-1] {$w_2 \atop \{f_t,f_d,a_d,a_g\}$};
\node[state,thick] (s-1-3) [right of=s-1-2] {$w_3 \atop \{f_t,f_g,a_t,a_d\}$};
\node[state] (s-1-4) [right of=s-1-3] {$w_4 \atop \{f_t,f_g,a_g,a_d\}$};
\node[state] (s-1-5) [right of=s-1-4] {$w_5 \atop \{f_d,f_g,a_d,a_t\}$};
\node[state] (s-1-6) [right of=s-1-5] {$w_6 \atop \{f_d,f_g,a_g,a_t\}$};

\end{tikzpicture}
\end{center}

\caption{A contextualized pointed model for Example \ref{example:tiger}}
\label{fig:A contextualized pointed model for the example of tiger}
\end{figure}
\end{example}

%%%%%%%%%%%%%%%%%%%%%%%%%%%%%%%%%%%%%%%%
\subsection{Remarks}
\label{subsec:Remarks}
%%%%%%%%%%%%%%%%%%%%%%%%%%%%%%%%%%%%%%%%

%\subsubsection{}

It can be seen that the truth value of $\cond{\alpha} \phi$ at $(\MM,\CC,w)$ is not dependent on $w$. So our semantics is global. This results from that contexts are not dependent on states. This coincides with an important feature of the weak ontic necessity: Whether it holds at a state is not dependent on the state.

%\subsubsection{}

Conditionals are not monotonic in the semantics. Here is a counter-example. Let $(\MM, \CC, w_1)$ be a contextualized pointed model, where

\begin{itemize}
\item $\MM = (W,V)$, where $W = \{w_1, w_2\}$, $V(p) = \{w_1,w_2\}$ and $V(q) = \{w_2\}$;
\item $\CC = (\DDD,\succ)$, where $\DDD = \{\{w_2\}\}$.
\end{itemize}

\noindent It can be verified $\MM, \CC, w_1 \Vdash \cond{p} q$ but $\MM, \CC, w_1 \not \Vdash \cond{p \land \neg q} q$.

By closely observing this example, we can see that the failure of monotonicity is due to the following reason: When evaluating $\cond{p} q$, the default in $\CC$ plays a role, but when evaluating $\cond{p \land \neg q} q$, it is defeated by the default $\defg{p \land \neg q}$.

%%%%%%%%%%%%%%%%%%%%%%%%%%%%%%%%%%%%%%%%
\subsection{An equivalent semantics}
\label{subsec:An equivalent semantics}
%%%%%%%%%%%%%%%%%%%%%%%%%%%%%%%%%%%%%%%%

%\subsubsection{}

Technically, we can use nonempty sequences of sets of states as contexts without changing the set of formulas. Let us be precise.

%\subsubsection{}

%%%%%%%%%%%%%%%%%%%%%%%%%%%%%%%%%%%%%%%%%%%%%
%%%%%%%%%%%%%%%%%%%%%%%%%%%%%%%%%%%%%%%%%%%%%
\begin{definition}[Contexts] \label{def:Contexts S}
Let $\MM = (W, V)$ be a model. A nonempty finite sequence $\SC = (\DD_0,\dots,$ $\DD_n)$ of defaults is called a context for $\MM$.
\end{definition}

%%%%%%%%%%%%%%%%%%%%%%%%%%%%%%%%%%%%%%%%%%%%%
%%%%%%%%%%%%%%%%%%%%%%%%%%%%%%%%%%%%%%%%%%%%%
\begin{definition}[Expected states by contexts] \label{def:Expected states by contexts S}
Let $\MM = (W,V)$ be a model, $\SC = (\DD_0,\dots,\DD_n)$ be a context. Define the set $\sete{\SC}$ of expected states by $\SC$ as follows:

\begin{itemize}
\item Suppose $\DD_0 = \emptyset$. Then $\sete{\SC} := \DD_0$.
\item Suppose $\DD_0 \neq \emptyset$. Then $\sete{\SC} := \DD_0 \cap \dots \cap \DD_k$, where $(\DD_0, \dots, \DD_k)$ is the longest initial segment of $(\DD_0, \dots, \DD_n)$ such that $\DD_0 \cap \dots \cap \DD_k \neq \emptyset$.
\end{itemize}
\end{definition}

We use $\SC_1 ; \SC_2$ to indicate the \emph{concatenation} of two contexts $\SC_1$ and $\SC_2$.

%%%%%%%%%%%%%%%%%%%%%%%%%%%%%%%%%%%%%%%%%%%%%
%%%%%%%%%%%%%%%%%%%%%%%%%%%%%%%%%%%%%%%%%%%%%
\begin{definition}[Update of contexts] \label{def:Update of contexts S}
Let $\SC$ be a context for a model $\MM = (W,V)$. Let $\DD$ be a subset of $W$. We define the update of $\SC$ with $\DD$ as the context $\updd{\SC}{\DD} = \DD ; \SC$.
\end{definition}

Other ingredients of the semantics are given as before.

%\subsubsection{}

\begin{definition}[Cores of contexts] \label{def:Cores of contexts S}
Let $\SC = (\DD_0,\dots,\DD_n)$ be a context for a model $\MM$. Recursively define $\SC^\tau$, the core of $\SC$, as follows:

\begin{itemize}
\item Let $\DD_0^\tau = \DD_0$;
\item[$\vdots$]
\item If $\DD_{k+1}$ does not occur in $(\DD_0,\dots,\DD_k)^\tau$, let $(\DD_0,\dots,\DD_k,\DD_{k+1})^\tau = (\DD_0,\dots,\DD_k)^\tau ; \DD_{k+1}$, or else let $(\DD_0,\dots,\DD_k,\DD_{k+1})^\tau = (\DD_0,\dots,\DD_k)^\tau$.
\end{itemize}
\end{definition}

Here is an example for $\SC^\tau$: If $\SC = (\DD_0, \DD_1, \DD_0, \DD_2, \DD_1)$, then $\SC^\tau = (\DD_0, \DD_1, \DD_2)$. One default can occur in a context for more than one times. Intuitively, the function $\tau$ just keeps the occurrence of a default in a context that is closest to the beginning of the context.

%%%%%%%%%%%%%%%%%%%%%%%%%%%%%%%%%%%%%%%%%%%%%
%%%%%%%%%%%%%%%%%%%%%%%%%%%%%%%%%%%%%%%%%%%%%
\begin{lemma} \label{lemma:sequence}
Let $\SC$ be a context for a model $\MM = (W,V)$ and $\DD \subseteq W$. Then:

\begin{enumerate}
\item $\sete{\SC} = \sete{\SC^\tau}$; \label{lemma-sequence: core is core}
\item $(\updd{\SC}{\DD})^\tau = (\updd{\SC^\tau}{\DD})^\tau$. \label{lemma-sequence: one two core}
\end{enumerate}
\end{lemma}

\begin{proof}
~

1. Let $\SC = (\DD_0,\dots,\DD_n)$ and $\SC^\tau = (\DD_{i_0}, \dots, \DD_{i_m})$. Assume $\DD_0 \cap \dots \cap \DD_n = \emptyset$. Then $\sete{\SC} = \emptyset = \sete{\SC^\tau}$. Assume $\DD_0 \cap \dots \cap \DD_n \neq \emptyset$. Let $l$ be the greatest number such that $\DD_0 \cap \dots \cap \DD_l \neq \emptyset$. Assume $l = n$. Then $\sete{\SC} = \DD_0 \cap \dots \cap \DD_n = \DD_{i_0} \cap \dots \cap \DD_{i_m} = \sete{\SC^\tau}$. Assume $l < n$. Then $\DD_{l+1}$ does not occur in $(\DD_0,\dots,\DD_l)$. Then $(\DD_0,\dots,\DD_l)^\tau$ is a proper initial segment of $\SC^\tau$. Let $(\DD_0,\dots,\DD_l)^\tau = (\DD_{i_0}, \dots, \DD_{i_k})$. Then $\DD_0 \cap \dots \cap \DD_l = \DD_{i_0} \cap \dots \cap \DD_{i_k}$ and $\DD_{i_{k+1}} = \DD_{l+1}$. Then $\DD_{i_0} \cap \dots \cap \DD_{i_k} \neq \emptyset$ and $\DD_{i_0} \cap \dots \cap \DD_{i_k} \cap \DD_{i_{k+1}} = \emptyset$. Then $\sete{\SC} = \DD_0 \cap \dots \cap \DD_l = \DD_{i_0} \cap \dots \cap \DD_{i_k} = \sete{\SC^\tau}$.

2. Assume $\DD$ does not occur in $\SC$. Then $(\updd{\SC}{\DD})^\tau = (\DD ; \SC)^\tau = \DD ; \SC^\tau = (\DD ; \SC^\tau)^\tau$. Assume $\DD$ occurs in $\SC$. Then $(\updd{\SC}{\DD})^\tau = (\DD ; \SC)^\tau = \DD ; (\SC - D)^\tau = \DD ; (\SC^\tau - \DD) = (\DD ; \SC^\tau)^\tau$.
\end{proof}

%%%%%%%%%%%%%%%%%%%%%%%%%%%%%%%%%%%%%%%%%%%%%
%%%%%%%%%%%%%%%%%%%%%%%%%%%%%%%%%%%%%%%%%%%%%
\begin{theorem}
For all $\phi \in \Phi_\conwon$, pointed model $(\MM,w)$, and contexts $\SC_1$ and $\SC_2$ such that $\SC_1^\tau = \SC_2^\tau$, $\MM,\SC_1,w \Vdash \phi$ if and only if $\MM,\SC_2,w \Vdash \phi$.
\end{theorem}

\begin{proof}
We put an induction on $\phi$. We only consider the case $\phi = \cond{\alpha} \psi$ and skip others. Assume $\MM,\SC_1,w \Vdash \cond{\alpha} \psi$. Then for every $u \in \sete{\updf{\SC_1}{\alpha}}$, $\MM,\updf{\SC_1}{\alpha},u \Vdash \psi$. By Item \ref{lemma-sequence: one two core} in Lemma \ref{lemma:sequence}, $(\updf{\SC_1}{\alpha})^\tau = (\updf{\SC_2}{\alpha})^\tau$. By Item \ref{lemma-sequence: core is core} in Lemma \ref{lemma:sequence}, $\sete{\updf{\SC_1}{\alpha}} = \sete{\updf{\SC_2}{\alpha}}$. Then for every $u \in \sete{\updf{\SC_2}{\alpha}}$, $\MM,\updf{\SC_2}{\alpha},u \Vdash \psi$. Then $\MM,\SC_2,w \Vdash \cond{\alpha} \psi$. The other direction is similar.
\end{proof}

Then the following result can be shown:

%%%%%%%%%%%%%%%%%%%%%%%%%%%%%%%%%%%%%%%%%%%%%
%%%%%%%%%%%%%%%%%%%%%%%%%%%%%%%%%%%%%%%%%%%%%
\begin{theorem}
The two semantics for $\Phi_\conwon$ determine the same set of valid formulas.
\end{theorem}

%\subsubsection{}

We will use this semantics when technical points are involved. We will use $\theta$ to indicate the special context $W$ for a model $\MM = (W,V)$.

%%%%%%%%%%%%%%%%%%%%%%%%%%%%%%%%%%%%%%%%
\subsection{Expressivity}
\label{subsec:Expressivity}
%%%%%%%%%%%%%%%%%%%%%%%%%%%%%%%%%%%%%%%%

Recall $\Phi_{\conwon\text{-}\mathsf{1}}$ is the flat fragment of $\Phi_\conwon$. In fact, $\Phi_\conwon$ is as expressive as $\Phi_{\conwon\text{-}\mathsf{1}}$.

\begin{definition}[Closed formulas] \label{def:Closed formulas}
Closed formulas of $\Phi_\conwon$ are defined as follows, where $\alpha \in \Phi_\PL$ and $\phi \in \Phi_\conwon$:

\[\chi ::= \cond{\alpha} \phi \mid \neg \chi \mid (\chi \land \chi)\]
\end{definition}

By the following fact, which is easy to verify, the truth value of a closed formula at a contextualized pointed model $(\MM,\CC,w)$ is independent of $w$.

\begin{fact} \label{fact:closed formulas}
Let $\chi$ be a closed formula. Let $\MM$ be a model and $\CC$ be a context. Then $\MM,\CC,w \Vdash \chi$ if and only if $\MM,\CC,u \Vdash \chi$ for all $w$ and $u$.
\end{fact}

%%%%%%%%%%%%%%%%%%%%%%%%%%%%%%%%%%%%%%%%%%%%%
%%%%%%%%%%%%%%%%%%%%%%%%%%%%%%%%%%%%%%%%%%%%%
\begin{lemma} \label{lemma:partial reduction conditionals valid}
The following formulas are valid, where $\alpha, \beta$ and $\gamma$ are in $\Phi_\PL$:
\begin{enumerate}
\item $\cond{\alpha} (\phi \land \psi) \leftrightarrow (\cond{\alpha} \phi \land \cond{\alpha} \psi)$ \label{validity:conditionals distribution conjunction}
\item $\cond{\alpha} (\phi \lor \chi) \leftrightarrow (\cond{\alpha} \phi \lor \cond{\alpha} \chi)$, where $\chi$ is a closed formula \label{validity:conditionals distribution conditionally distribution}
\item $\cond{\alpha} \cond{\beta} \gamma \leftrightarrow (\EEE\alpha \rightarrow ((\EEE(\alpha \land \beta) \land \cond{\alpha \land \beta} \gamma) \lor (\neg \EEE(\alpha \land \beta) \land \AAA(\beta \rightarrow \gamma))))$ \label{validity:conditionals conditionals}
\item $\cond{\alpha} \dcon{\beta} \gamma \leftrightarrow (\EEE\alpha \rightarrow ((\EEE(\alpha \land \beta) \land \dcon{\alpha \land \beta} \gamma) \lor (\neg \EEE(\alpha \land \beta) \land \EEE(\beta \land \gamma))))$ \label{validity:conditionals dual conditionals}
\end{enumerate}
\end{lemma}

\noindent The proof for this result can be found in Section \ref{subsec:Proofs about expressivity of ConWON} in the appendix.

%%%%%%%%%%%%%%%%%%%%%%%%%%%%%%%%%%%%%%%%%%%%%%%%%
%%%%%%%%%%%%%%%%%%%%%%%%%%%%%%%%%%%%%%%%%%%%%%%%%
\begin{theorem} \label{theorem:reduction conwon validity}
There is an effective function $\sigma$ from $\Phi_{\conwon}$ to $\Phi_{\conwon\text{-}\mathsf{1}}$ such that for every $\phi \in \Phi_{\conwon}$, $\phi \leftrightarrow \sigma(\phi)$ is valid.
\end{theorem}

\noindent The proof for this result can be found in Section \ref{subsec:Proofs about expressivity of ConWON} in the appendix.

%%%%%%%%%%%%%%%%%%%%%%%%%%%%%%%%%%%%%%%%
\subsection{Axiomatization}
\label{subsec:Axiomatization}
%%%%%%%%%%%%%%%%%%%%%%%%%%%%%%%%%%%%%%%%

%%%%%%%%%%%%%%%%%%%%%%%%%%%%%%%%%%%%%%%%
%%%%%%%%%%%%%%%%%%%%%%%%%%%%%%%%%%%%%%%%
\begin{definition}[Axiomatic system $\conwon$] \label{def:Axiomatic system ConWON}

Define an axiomatic system $\conwon$ as follows:

\noindent Axioms:

\begin{enumerate}
%%%%
\item Axioms for the Propositional Logic
%%%%
\item Axioms for partial reduction of $\cond{\cdot}$, where $\alpha, \beta$ and $\gamma$ are in $\Phi_\PL$:
\begin{enumerate}
\item $\cond{\alpha} (\phi \land \psi) \leftrightarrow (\cond{\alpha} \phi \land \cond{\alpha} \psi)$ \label{axiom:conditionals distribution conjunction}
\item $\cond{\alpha} (\phi \lor \chi) \leftrightarrow (\cond{\alpha} \phi \lor \cond{\alpha} \chi)$, where $\chi$ is a closed formula \label{axiom:conditionals distribution conditionally distribution}
\item $\cond{\alpha} \cond{\beta} \gamma \leftrightarrow (\EEE\alpha \rightarrow ((\EEE(\alpha \land \beta) \land \cond{\alpha \land \beta} \gamma) \lor (\neg \EEE(\alpha \land \beta) \land \AAA(\beta \rightarrow \gamma))))$ \label{axiom:conditionals conditionals}
\item $\cond{\alpha} \dcon{\beta} \gamma \leftrightarrow (\EEE\alpha \rightarrow ((\EEE(\alpha \land \beta) \land \dcon{\alpha \land \beta} \gamma) \lor (\neg \EEE(\alpha \land \beta) \land \EEE(\beta \land \gamma))))$ \label{axiom:conditionals dual conditionals}
\end{enumerate}
%%%%
\item Axioms for $\cond{\alpha} \beta$, where $\alpha$ and $\beta$ are in $\Phi_\PL$:
\begin{enumerate}
\item $\cond{\alpha} \alpha$
\item $\cond{\alpha} \gamma \rightarrow \cond{\alpha} (\gamma \lor \delta)$
\item $(\cond{\alpha} \beta \land \cond{\alpha} \gamma) \rightarrow \cond{\alpha \land \beta} \gamma$
\item $(\cond{\alpha} \gamma \land \cond{\beta} \gamma) \rightarrow \cond{\alpha \lor \beta} \gamma$
\item $(\dcon{\alpha} \beta \land \cond{\alpha} \gamma) \rightarrow \cond{\alpha \land \beta} \gamma$
\end{enumerate}
%%%%
\end{enumerate}

\noindent Inference rules:

\begin{enumerate}
\item Modus ponens: From $\phi$ and $\phi \rightarrow \psi$, we can get $\psi$;
\item Replacement of equivalent antecedents: From $\alpha \leftrightarrow \beta$, we can get $\cond{\alpha} \gamma \leftrightarrow \cond{\beta} \gamma$, where $\alpha, \beta$ and $\gamma$ are in $\Phi_\PL$;
\item Replacement of equivalent consequents: From $\gamma \leftrightarrow \delta$, we can get $\cond{\alpha} \gamma \leftrightarrow \cond{\alpha} \delta$, where $\gamma$ and $\delta$ are in $\Phi_\PL$.
\end{enumerate}
\end{definition}

Define \emph{derivability} with respect to $\conwon$ as usual. We use $\vdash_\conwon \phi$ to indicate that $\phi$ is derivable in $\conwon$.

Let $\sigma$ be the function from $\Phi_{\conwon}$ to $\Phi_{\conwon\text{-}\mathsf{1}}$ defined in the proof for Theorem \ref{theorem:reduction conwon validity}. From the proof for Theorem \ref{theorem:reduction conwon validity} and the definition of $\conwon$, we can get the following result:

%%%%%%%%%%%%%%%%%%%%%%%%%%%%%%%%%%%%%%%%%%%%%%%%%
%%%%%%%%%%%%%%%%%%%%%%%%%%%%%%%%%%%%%%%%%%%%%%%%%
\begin{lemma} \label{lemma:reduction conwon}
For every $\phi \in \Phi_{\conwon}$, $\phi \leftrightarrow \sigma(\phi)$ is derivable in $\conwon$.
\end{lemma}

Later we will show that the flat fragment of $\Phi_\conwon$ shares the same valid formulas with the flat fragment of the conditional logic $\logv$ proposed in Lewis \cite{Lewis73}. Burgess \cite{Burgess81} provided a complete axiomatic system for $\logv$. It can be shown that the system restricted to the flat fragment of $\logv$ is complete. The restricted system is contained in $\conwon$. Then we can get the completeness of $\conwon$.

%%%%%%%%%%%%%%%%%%%%%%%%%%%%%%%%%%%%%%%%%%%%%%%%%
%%%%%%%%%%%%%%%%%%%%%%%%%%%%%%%%%%%%%%%%%%%%%%%%%
\begin{theorem} \label{theorem:soundness completeness conwon}
The axiomatic system $\conwon$ is sound and complete with respect to the set of valid formulas of $\Phi_\conwon$.
\end{theorem}

\begin{proof}
By Lemma \ref{lemma:partial reduction conditionals valid}, Theorems \ref{theorem:valid in acl valid in conwon} and \ref{theorem:completeness acl}, we can get the soundness of $\conwon$. Consider the completeness of $\conwon$. Let $\phi$ be a valid formula in $\Phi_\conwon$. By Theorem \ref{theorem:reduction conwon validity}, $\sigma(\phi)$ is valid in $\conwon$. By Theorem \ref{theorem:valid in acl valid in conwon}, $\sigma(\phi)$ valid in $\logv$. By Theorem \ref{theorem:completeness acl-1}, $\sigma(\phi)$ is derivable in $\logv\text{-}\mathsf{1}$. Note $\logv\text{-}\mathsf{1}$ is contained in $\conwon$. Then $\sigma(\phi)$ is derivable in $\conwon$. By Lemma \ref{lemma:reduction conwon}, $\phi$ is derivable in $\conwon$.
\end{proof}

%%%%%%%%%%%%%%%%%%%%%%%%%%%%%%%%%%%%%%%%
\section{Comparisons}
\label{sec:Comparisons}
%%%%%%%%%%%%%%%%%%%%%%%%%%%%%%%%%%%%%%%%

As mentioned before, there are not theories explicitly handling conditional weak ontic necessity in the literature yet. In this section, we compare our theory to the following works on general conditionals, which are closely related to our theory: Stalnaker and Lewis's ordering semantics on counterfactual conditionals, Lewis's conditional logic $\logv$, Kratzer's premise semantics for counterfactual conditionals, Kratzer's semantics for conditional modalities, and Veltman's update semantics for counterfactual conditionals. The comparison is mainly from the following perspective: Where do they differ in handling conditional weak ontic necessity?

Conceptually, our theory is close to Kratzer's premise semantics and Veltman's update semantics but different from Stalnaker and Lewis's ordering semantics. Technically, the flat fragment of our logic is identical to the flat fragment of the logic $\logv$. They are different when nested conditionals are involved. As mentioned above, our semantics is global and this coincides with whether the weak ontic necessity holds at a state is not dependent on the state. All these theories are local.

%%%%%%%%%%%%%%%%%%%%%%%%%%%%%%%%%%%%%%%%
\subsection{Conceptual comparisons to Stalnaker and Lewis's ordering semantics on counterfactual conditionals}
\label{subsec:Conceptual comparisons to Stalnaker and Lewis's ordering semantics on counterfactual conditionals}
%%%%%%%%%%%%%%%%%%%%%%%%%%%%%%%%%%%%%%%%

%\subsubsection{}

Stalnaker and Lewis's ordering semantics was proposed in \cite{Stalnaker68} and \cite{Lewis73}. In both the ordering semantics and ours, evaluation of ``\exas{If $\phi$ then $\Box \psi$}'' at a possibility $x$ involves a check of satisfaction of $\psi$ at the possibilities in \emph{certain domain} with respect to $x$ and $\phi$. We call this domain the \emph{counterfactual domain} with respect to $x$ and $\phi$.

By the ordering semantics, a conditional ``\exas{If $\phi$ then $\Box \psi$}'' is true at a possibility $x$ if and only if $\psi$ is true at all the elements of $\Delta$, which consists of the possibilities that are most similar to $x$ among the alternatives of $x$ where $\phi$ is true. Here $\Delta$ is the counterfactual domain with respect to $x$ and $\phi$.

In our theory, a conditional ``\exas{If $\phi$ then $\Box \psi$}'' is true at a possibility $x$ relative to a context $\CC$ if and only if $\psi$ is true at all the possibilities in $\sete{\updf{\CC}{\phi}}$ relative to the context $\updf{\CC}{\phi}$. Here $\sete{\updf{\CC}{\phi}}$ is the counterfactual domain with respect to $x$ and $\phi$.

Conceptually, our work is different from Stalnaker and Lewis's theory in two aspects. Firstly, counterfactual domains are determined in different ways in the two theories. Secondly, our theory uses contexts, while Stalnaker and Lewis's theory does not.

%\subsubsection{}

We consider the first difference. Fix a counterfactual conditional ``\exas{If $\phi$ then $\Box \psi$}'' and a possibility $x$.

According to Stalnaker and Lewis's theory, the counterfactual domain with respect to $x$ and $\phi$ is determined as follows. Firstly, we get the set of all alternatives of $x$ where $\phi$ is true. Secondly, we choose from this set those elements that are most similar to $x$. The chosen elements form the counterfactual domain with respect to $x$ and $\phi$. Note what $x$ is like matters for determining the counterfactual domain.

In our theory, the counterfactual domain with respect to $x$ and $\phi$ is determined as follows. Firstly, we add the following default to the context and give it the highest priority: \emph{$\phi$ should be the case}. Secondly, we get the possibilities expected by the new context, which form the counterfactual domain with respect to $x$ and $\phi$. Note that what $x$ is like plays no role here.

The two different ways of determining counterfactual domains have different consequences in handling counterfactual conditionals. We illustrate this by an example adapted from \cite{Fine75}.

%%%%%%%%%%%%%%%%%%%%%%%%%%%%%%%%%%%%%%%%
%%%%%%%%%%%%%%%%%%%%%%%%%%%%%%%%%%%%%%%%
\begin{example} \label{example:Nixon}
Suppose now it is the time of Nixon as the president of the United States. There has been no nuclear holocaust in our world (up to the Nixon time). Look at the following sentence:

\bee
\ex \exas{If Nixon had pressed the button yesterday, then there \textbf{should} have been a nuclear holocaust.} \label{exe:If Nixon had pressed the button yesterday}
\eee
\end{example}

\noindent Intuitively, this sentence is true.

It seems that among the alternatives to the present state where Nixon pressed the button yesterday, those where there has been no nuclear holocaust resemble the present state the most. Therefore, by Stalnaker and Lewis's theory, the counterfactual domain with respect to the present state and ``\exas{Nixon pressed the button yesterday}'' consists of the alternatives to the present state where Nixon pressed the button yesterday but there has been no nuclear holocaust. As a result, Sentence \ref{exe:If Nixon had pressed the button yesterday} is false.

By our theory, that sentence is true. How? It seems reasonable to think that the context contains the following default: \emph{The command system of the United States works well}. Add \emph{Nixon pressed the button yesterday} as a default to the context. The set of expected states by the new context is the counterfactual domain with respect to the present state and ``\exas{Nixon pressed the button yesterday}''. At all the elements of the domain, there has been a nuclear holocaust. Thus, Sentence \ref{exe:If Nixon had pressed the button yesterday} is true.

%\subsubsection{}

The second difference makes a difference for nested conditionals. We consider an example adapted from \cite{McGee85}.

%%%%%%%%%%%%%%%%%%%%%%%%%%%%%%%%%%%%%%%%
%%%%%%%%%%%%%%%%%%%%%%%%%%%%%%%%%%%%%%%%
\begin{example} \label{example:Reagan not won}
Suppose now it is just after the 1980 election of the United States and Reagan has won. Before the election, Reagan and Anderson were the only two Republic candidates and Carter was a Democratic candidate. Opinion polls showed that Reagan was decisively ahead of the other candidates and Carter was far ahead of the other candidates except Reagan. Here is a nested conditional:

\bee
\ex \exas{If Reagan hadn't won the election, then if a Republican had won, it \textbf{should} have been Reagan.}
\eee
\end{example}

\noindent This sentence is clearly false.

However, by Stalnaker and Lewis's theory, this sentence is true. We cite an argument for this from \cite{McGee85}. 

%%%%%%%%%%%%%%%%%%%%%%%%%%%%%%%%%%%%%%%%
%%%%%%%%%%%%%%%%%%%%%%%%%%%%%%%%%%%%%%%%
\begin{quote}
However, the possible world most similar to the actual world in which Reagan did not win the election will be a world in which Carter finished first and Reagan second, with Anderson again a distant third, and so a world in which ``\exas{If a Republican had won it would have been Reagan}'' is true. Thus Stalnaker's theory wrongly predicts that, in the actual world, ``\exas{If Reagan hadn't won the election, then if a Republican had won, it would have been Reagan}'' will be true.
\end{quote}

Our theory has no problem with this example.

%%%%%%%%%%%%%%%%%%%%%%%%%%%%%%%%%%%%%%%%
%%%%%%%%%%%%%%%%%%%%%%%%%%%%%%%%%%%%%%%%
\begin{example} \label{example:Reagan not won modeling}
This example explains how the sentence ``\exas{If Reagan hadn't won the election, then if a Republican had won, it should have been Reagan}'' in Example \ref{example:Reagan not won} is analyzed in our formalization.

Note the result of opinion polls was that Reagan was decisively ahead of the other candidates and Carter was far ahead of the other candidates except Reagan. Two defaults can represent the result: \emph{Reagan or Carter will win} and \emph{Reagan will win}, where the former is prior to the latter.

We use $r,a$ and $c$ to respectively indicate ``\exas{Reagan wins the election}'', ``\exas{Anderson wins the election}'' and ``\exas{Carter wins the election}''.

Figure \ref{fig:after election} indicates a contextualized pointed model $(\MM, \CC, w_1)$ for Example \ref{example:Reagan not won}, where $\CC = (\DDD, \succ)$, where $\DDD$ has two defaults: $\DD_1 = \{w_1, w_3\}$ and $\DD_2 = \{w_1\}$, and $\DD_1 \succ \DD_2$.

\begin{figure}[h]
\begin{center}
\begin{tikzpicture}[node distance=20mm,scale=0.85,every node/.style={transform shape}]
\tikzstyle{every state}=[draw=black,text=black,minimum size=10mm]

\draw[] (-15mm,-15mm) rectangle (75mm,15mm);

\node[state,thick] (s-1-1) [] {$w_1 \atop \{r\}$};
\node[state] (s-1-2) [right of=s-1-1] {$w_2 \atop \{a\}$};
\node[state] (s-1-3) [right of=s-1-2] {$w_3 \atop \{c\}$};
\node[state] (s-1-4) [right of=s-1-3] {$w_4 \atop \{\}$};

\end{tikzpicture}
\end{center}
\caption{A contextualized pointed model for Example \ref{example:Reagan not won}}
\label{fig:after election}
\end{figure}
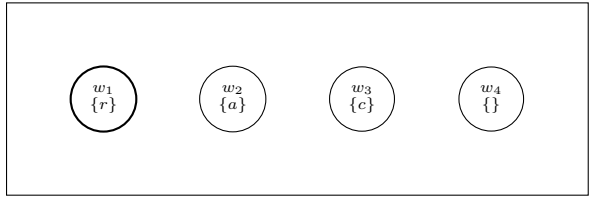

The nested conditional is translated as $\cond{\neg r} \cond{r \lor a} r$. It is false at $(\MM, \CC, w_1) $. How?

\begin{itemize}
\item $\updf{\CC}{\neg r}$, the update of $\CC$ with $\neg r$, is such that $\sete{\updf{\CC}{\neg r}}$, the set of expected states by $\updf{\CC}{\neg r}$, equals to $\{w_3\}$.
\item $\updf{(\updf{\CC}{\neg r})}{(r \lor a)}$, the update of $\updf{\CC}{\neg r}$ with $r \lor a$, is such that $\sete{\updf{(\updf{\CC}{\neg r})}{(r \lor a)}}$, the set of expected states by $\updf{(\updf{\CC}{\neg r})}{(r \lor a)}$, equals to $\{w_2\}$.
\item $\MM, \updf{(\updf{\CC}{\neg r})}{(r \lor a)}, w_2 \not \Vdash r$. Then $\MM, \CC, w_1 \not \Vdash \cond{\neg r} \cond{r \lor a} r$.
\end{itemize}
\end{example}

%%%%%%%%%%%%%%%%%%%%%%%%%%%%%%%%%%%%%%%%
\subsection{Technical comparisons to Lewis's conditional logic $\logv$}
\label{subsec:Technical comparisons to Lewis's conditional logic V}
%%%%%%%%%%%%%%%%%%%%%%%%%%%%%%%%%%%%%%%%

%\subsubsection{}

Lewis \cite{Lewis73} presented a class of conditional logics based on the ordering semantics and the smallest one is called $\logv$, standing for ``variable strictness''. In what follows, we first present this logic briefly and then compare it with ours.

%\subsubsection{}

%%%%%%%%%%%%%%%%%%%%%%%%%%%%%%%%%%%%%%%%%%%%%%%%%%%%
%%%%%%%%%%%%%%%%%%%%%%%%%%%%%%%%%%%%%%%%%%%%%%%%%%%%
\begin{definition}[Language $\Phi_\logv$] \label{def:Language V}
The language $\Phi_\logv$ is defined as follows:

\[\phi ::= p \mid \bot \mid \neg \phi \mid (\phi \land \phi) \mid (\phi \cona \phi)\]
\end{definition}

\noindent The formula $\phi \cona \psi$ is for conditionals. Define $\mathit{A} \phi$ as $\neg \phi \cona \bot$, meaning \emph{$\phi$ is necessary}. Define $\mathit{E} \phi$ as $\neg \mathit{A} \neg \phi$, meaning \emph{$\phi$ is possible}.

%%%%%%%%%%%%%%%%%%%%%%%%%%%%%%%%%%%%%%%%%%%%%%%%%%%%
%%%%%%%%%%%%%%%%%%%%%%%%%%%%%%%%%%%%%%%%%%%%%%%%%%%%
\begin{definition}[Relational models for $\Phi_\logv$] \label{def:Relational models for Phi V}
A tuple $M = (W, \Gamma, V)$ is a relational model for $\Phi_\logv$ if

\begin{itemize}
\item $W$ and $V$ are as usual;
\item For every $w \in W$, $\Gamma(w)$ is a tuple $(W_w,<_w)$, where $W_w \subseteq W$ and $<_w$ is an irreflexive, transitive, and almost connected binary relation on $W_w$\footnote{A binary relation $<$ is almost connected if for all $w,u$ and $v$, if $w < u$, then $w< v$ or $v < u$.}.
\end{itemize}
\end{definition}

%%%%%%%%%%%%%%%%%%%%%%%%%%%%%%%%%%%%%%%%%%%%%%%%%%%%
%%%%%%%%%%%%%%%%%%%%%%%%%%%%%%%%%%%%%%%%%%%%%%%%%%%%
\begin{definition}[Relational semantics for $\Phi_\logv$] \label{def:Relational semantics for Phi V}
Let $M = (W, \Gamma, V)$ be a relational model.

\begin{tabular}{lll}
$M,w \Vdash \phi \cona \psi$ & $\Leftrightarrow$ & \parbox[t]{27em}{for every $u \in W_w$, if $M,u \Vdash \phi$, then there is $v \in W_w$ such that $M,v \Vdash \phi$ and for every $z \in W_w$, if $z <_w v$, then $M,z \Vdash \psi$}
\end{tabular}
\end{definition}

%%%%%%%%%%%%%%%%%%%%%%%%%%%%%%%%%%%%%%%%%%%%%%%%%%%%
%%%%%%%%%%%%%%%%%%%%%%%%%%%%%%%%%%%%%%%%%%%%%%%%%%%%
\begin{definition}[Axiomatic system $\logv$] \label{def:Axiomatic system V}
Define an axiomatic system $\logv$ as follows:

\noindent Axioms:

\begin{enumerate}
\item Axioms for the Propositional Logic
\item Axioms for $\cona$:
\begin{enumerate}
\item $\phi \cona \phi$
\item $((\phi \cona \chi) \land (\phi \cona \xi)) \rightarrow (\phi \cona (\chi \land \xi))$
\item $(\phi \cona \chi) \rightarrow (\phi \cona (\chi \lor \xi))$
\item $((\phi \cona \psi) \land (\phi \cona \chi)) \rightarrow ((\phi \land \psi) \cona \chi)$
\item $((\phi \cona \chi) \land (\psi \cona \chi)) \rightarrow ((\phi \lor \psi) \cona \chi)$
\item $(\neg (\phi \cona \neg \psi) \land (\phi \cona \chi)) \rightarrow ((\phi \land \psi) \cona \chi)$
\end{enumerate}
\end{enumerate}

\noindent Inference rules:

\begin{enumerate}
\item Modus ponens: From $\phi$ and $\phi \rightarrow \psi$, we can get $\psi$;
\item Replacement of equivalent antecedents: From $\phi \leftrightarrow \psi$, we can get $(\phi \cona \chi) \leftrightarrow (\psi \cona \chi)$;
\item Replacement of equivalent consequents: From $\chi \leftrightarrow \xi$, we can get $(\phi \cona \chi) \leftrightarrow (\phi \cona \xi)$.
\end{enumerate}
\end{definition}

\begin{theorem} \label{theorem:completeness acl}
The axiomatic system $\logv$ is sound and complete with respect to the set of valid formulas in $\Phi_\logv$.
\end{theorem}

\begin{theorem} \label{theorem:finite model property}
The language $\Phi_\logv$ has the finite model property.
\end{theorem}

\noindent The proofs for the two results can be found in Burgess \cite{Burgess81}.

Since $\Phi_\logv$ has the finite model property, we can assume $<_w$ is well-founded without changing the set of valid formulas. Then the semantics could be simplified as follows:

\medskip

\begin{tabular}{lll}
$M,w \Vdash \phi \cona \psi$ & $\Leftrightarrow$ & \parbox[t]{27em}{$M,u \Vdash \psi$ for every $<_w$-minimal element $u$ of $\stas{\phi}$}
\end{tabular}

\noindent where $\stas{\phi} = \{x \mid M,x \Vdash \phi\}$.

\medskip

%\subsubsection{}

We use $\Phi_{\logv\text{-}\mathsf{1}}$ to indicate the flat fragment of $\Phi_\logv$ and use $\logv\text{-}\mathsf{1}$ to indicate the restriction of the system $\logv$ to $\Phi_{\logv\text{-}\mathsf{1}}$.

Fix a formula $\phi$ in $\Phi_{\logv\text{-}\mathsf{1}}$ and a proof $\psi_0, \dots, \psi_n$ for $\phi$ with respect to $\logv$. It can be shown that there is a proof $\psi_{i_0}, \dots, \psi_{i_l}$ for $\phi$ such that its elements are in $\{\psi_0, \dots, \psi_n\} \cap \Phi_{\logv\text{-}\mathsf{1}}$. Then we can get the following result:

\begin{theorem} \label{theorem:completeness acl-1}
The axiomatic system $\logv\text{-}\mathsf{1}$ is sound and complete with respect to the set of valid formulas in $\Phi_{\logv\text{-}\mathsf{1}}$.
\end{theorem}

%\subsubsection{}

The logics $\conwon$ and $\logv$ share the same flat fragments:

\begin{theorem} \label{theorem:valid in acl valid in conwon}
For every $\phi \in \Phi_{\logv\text{-}\mathsf{1}}$, $\phi$ is valid in $\logv$ if and only if $\phi$ is valid in $\conwon$.
\end{theorem}

\noindent The proof for this result can be found in Section \ref{subsec:Proofs about comparisons to the conditional logic V} in the appendix.

%\subsubsection{}

The logics $\conwon$ and $\logv$ are different when nested conditionals are involved, as the following fact shows:

\begin{fact} \label{fact: nested conditionals conwon v}
~
\begin{enumerate}
\item $\EEE (p \land q) \rightarrow \cond{p}\cond{q} (p \land q)$ is valid in $\conwon$.
\item $\mathit{E} (p \land q) \rightarrow (p \cona (q \cona (p \land q)))$ is invalid in $\logv$.
\end{enumerate}
\end{fact}

\begin{proof}
~

1. Assume that $\MM,\CC,w \Vdash \EEE (p \land q)$ but $\MM,\CC,w \not \Vdash \cond{p}\cond{q} (p \land q)$. Then there is $u \in \sete{\updf{\CC}{p}}$ such that $\MM,\updf{\CC}{p},u \not \Vdash \cond{q}(p\land q)$. Then there is $v \in \sete{\updf{(\updf{\CC}{p})}{q}}$ such that $\MM,\updf{(\updf{\CC}{p})}{q},v \not \Vdash p\land q$. Then $v \notin \defg{p\land q}$. By Items \ref{item:basic alpha plus beta alpha and beta} and \ref{item:basic alpha alpha} in Lemma \ref{lemma:basic}, $\sete{\updf{(\updf{\CC}{p})}{q}} \subseteq \defg{p\land q}$. Then $v \in \defg{p\land q}$. We have a contradiction.

2. Let $M = (W,\Gamma,V)$ be a relational model for $\Phi_\logv$, where $W = \{w_1,w_2\}$, $\Gamma(w_1) = \Gamma(w_2) = (W,<)$, where $w_2 < w_1$, $V(p) =\{w_1\}$ and $V(q) = \{w_1, w_2\}$. We can see $M,w_1\Vdash \mathit{E} (p \land q)$. Note $w_2$ is a $<$-minimal $q$-state but $M,w_2 \not \Vdash p\land q$. Then $M,w_1\not\Vdash q \cona (p\land q)$. Note that $w_1$ is a minimal $p$-state. Then $M,w_1\not\Vdash p \cona (q \cona (p\land q))$.
\end{proof}

%%%%%%%%%%%%%%%%%%%%%%%%%
\subsection{Comparisons to Kratzer's premise semantics for counterfactual conditionals}
\label{subsec:Comparisons to ratzer's premise semantics for counterfactual conditionals}
%%%%%%%%%%%%%%%%%%%%%%%%%

%\subsubsection{}

Kratzer \cite{Kratzer79} presented a premise semantics for counterfactual conditionals, which is developed further in Kratzer \cite{Kratzer81}. The following quote, which is from Kratzer \cite{Kratzer81}, shows the core idea of this semantics:

%%%%%%%%%%%%%%%%%%%%%%%%%%%%%%%%%%%%%%%%
%%%%%%%%%%%%%%%%%%%%%%%%%%%%%%%%%%%%%%%%
\begin{quote}
The truth of counterfactuals depends on everything which is the case in the world under consideration: in assessing them, we have to consider all the possibilities of adding as many facts to the antecedent as consistency permits. If the consequent follows from every such possibility, then (and only then), the whole counterfactual is true.
\end{quote}

%\subsubsection{}

In this theory, the evaluation context for formulas is a tuple $(W, \Gamma, V, w)$, where $W$ and $V$ are as usual, $w$ is a state in $W$, and for every $w \in W$, $\Gamma(w)$ is a set of subsets of $W$. Intuitively, the elements of $\Gamma(w)$ are propositions, indicating premises for evaluating counterfactual conditionals at $w$. Constraints can be introduced to $\Gamma$. For example, it can be required that for every $X \in \Gamma(w)$, $w \in X$, which means that all the propositions in $\Gamma(w)$ are true at $w$.

Let $\Delta$ be a set of sets of states and $X$ be a set of states. Define $\mathbf{K} (\Delta, X)$ as $\{\Delta' \cup \{X\} \mid \text{$\Delta'$ is a maximal subset of $\Delta$ such that $\bigcap (\Delta' \cup \{X\}) \neq \emptyset$}\}$.

A conditional ``\exas{If $\phi$ then $\Box \psi$}'' is true at $(W, \Gamma, V, w)$ if and only if for every $\Lambda$ in $\mathbf{K} (\Delta, |\phi|)$, $\bigcap \Lambda \subseteq |\psi|$, where $|\phi|$ and $|\psi|$ are sets of states respectively satisfying $\phi$ and $\psi$.

%\subsubsection{}

Lewis \cite{Lewis81} shows that technically, premise semantics is equivalent to ordering semantics: Every set of premises at a state $w$ can induce an equivalent ordering at $w$, and for every ordering at $w$, there is an equivalent set of premises at $w$.

%\subsubsection{}

In our theory, the evaluation context is $(W,V,\CC,w)$. It shares similar ideas with $(W, \Gamma, V, w)$: Evaluation of conditionals involves additional premises. The main difference is that we introduce an order between premises.

Our semantics is dynamic in the following sense: When evaluating conditionals, states' premises might change. Kratzer's semantics is not dynamic in this sense. As a consequence, nested conditionals are handled differently in the two semantics. We look at an example. As mentioned above, the formula $\EEE (p \land q) \rightarrow \cond{p}\cond{q} (p \land q)$ is valid in our semantics. However, just like the ordering semantics, the following sentence is not valid in Krazter's theory: \emph{Given $p$ and $p$ could be true at the same time, if $p$, then if $q$, then $p$ and $q$}. The invalidity of this sentence can be shown by a pointed premise model transformed from the pointed ordering model in Fact \ref{fact: nested conditionals conwon v}, which shows the invalidity of $\mathit{E} (p \land q) \rightarrow (p \cona (q \cona (p \land q)))$ in $\logv$.

%%%%%%%%%%%%%%%%%%%%%%%%%
\subsection{Comparisons to Kratzer's semantics for conditional modalities}
\label{subsec:Comparisons to Kratzer's semantics for conditional modalities}
%%%%%%%%%%%%%%%%%%%%%%%%%

Influenced by Lewis \cite{Lewis81}, Kratzer \cite{Kratzer91} presented a general approach to deal with modalities. This work contains a way to handle conditional modalities.

In this work, the evaluation context for sentences is a tuple $(W,f,g,V,w)$, where $W$ and $V$ are as usual, $w$ is a state in $W$, and $f$ and $g$ are functions from $W$ to $\mathcal{P} (\mathcal{P} (W))$, respectively called the modal base and the ordering source. In order to ease the comparisons, we assume $W$ is finite.

For every $w$ of $W$, $\bigcap f(w)$ is a set of states, intuitively understood as the set of accessible states to $w$. For every $w$ of $W$, $g(w)$ is a set of sets, which induces a binary relation $\leq_w$ on $W$ in the following way: For all $u$ and $v$, $u \leq_w v$ if and only if for all $X \in g(w)$, if $v \in X$, then $u \in X$. Intuitively, $u \leq_w v$ means that $u$ is at least as optimal as $v$.

Many modalities can be interpreted in the semantic settings.

A necessity $\Box \phi$ is true at $(W,f,g,V,w)$ if and only if for all $\leq_w$-minimal elements $u$ of $\bigcap f(w)$, $\phi$ is true at $(W,f,g,V,u)$.

A conditional necessity ``\exas{If $\phi$ then $\Box \psi$}'' is true at $(W,f,g,V,w)$ if and only if for all $\leq_w$-minimal elements $u$ of $|\phi| \cap \bigcap f(w)$, $\psi$ is true at $(W,f,g,V,u)$, where $|\phi|$ is the set of possibilities meeting $\phi$.

Our work differs from Kratzer's semantics for conditional necessities in the following aspects. Firstly, Kratzer's theory uses two things to deal with conditional necessities, that is, a modal base and an ordering source. Our theory just uses one thing, that is, a context. Secondly, minimal elements are directly used in Kratzer's theory but not in ours. Thirdly, in Kratzer's theory, the if-clause of a conditional necessity restricts the domain of accessible states. Consequently, conditionals are monotonic. In our theory, the if-clause of conditional weak ontic necessity updates contexts. As mentioned above, conditionals are not monotonic in our theory.

%%%%%%%%%%%%%%%%%%%%%%%%%
\subsection{Comparisons to Veltman's update semantics for counterfactual conditionals}
\label{subsec:Comparisons to Veltman's update semantics for counterfactual conditionals}
%%%%%%%%%%%%%%%%%%%%%%%%%

In \cite{Veltman05}, Veltman provides an update semantics for counterfactual conditionals, which combines the approaches from his earlier work \cite{Veltman76} and from Kratzer \cite{Kratzer81}.

One important notion of update semantics is \emph{support}: An information state supports a sentence if the information conveyed by the sentence is already contained in the state.

In \cite{Veltman05}, an \emph{information state} is defined as a tuple $\langle F, U \rangle$, where $F$ and $U$ are two sets of possible worlds such that $F \subseteq U$. A world $w$ is in $F$ if, for all the agent knows, $w$ might be the actual world. $U$ consists of all the possible worlds that meet all the general laws accepted by the agent.

Counterfactual conditionals are handled in the following way: An information state supports ``\exas{If $\phi$ then $\Box \psi$}'' if and only if $\Box \psi$ is supported by the result of updating the information state with ``\exas{If $\phi$}''. Here we do not go through how an information state is updated.

Our semantics follows this approach's core idea.

%%%%%%%%%%%%%%%%%%%%%%%%%%%%%%%%%%%%%%%%
\section{Looking backward and forward}
\label{sec:Looking backward and forward}
%%%%%%%%%%%%%%%%%%%%%%%%%%%%%%%%%%%%%%%%

In this work, we present a formalization for conditional weak ontic necessity. It has the following features. It introduces contexts, which are sets of ordered defaults. Contexts determine which possible states of the present world are expected. Conditionals are evaluated relative to contexts and their if-clauses change contexts.

There is some work worth doing in the future. As stated previously, we understand the conditional weak ontic necessity from a temporal perspective. However, our formalization is just a slice of the time flow and cannot handle its temporal dimension genuinely. For example, the following three sentences have clear connections, but our formalization cannot capture them.

\bee
\exr{exe:If the goat is still alive tomorrow} \exas{If the goat is still alive tomorrow, the dog \textbf{should} be dead.}
\exr{exe:If the goat is still alive now} \exas{If the goat is still alive now, the dog \textbf{should} be dead.}
\exr{exe:If the goat were still alive now} \exas{If the goat were still alive now, the dog \textbf{should} be dead.}
\eee

\noindent We leave the introduction of temporality as future work.

\bibliographystyle{alpha}
\bibliography{Conditionals}

\appendix

%%%%%%%%%%%%%%%%%%%%%%%%%%%%%%%%%%%%%%%%
\section{Proofs}
\label{sec:Proofs}
%%%%%%%%%%%%%%%%%%%%%%%%%%%%%%%%%%%%%%%%

\newcommand{\cpa}{\sete{\updf{\CC}{\alpha}}}

\newcommand{\mcwt}{\MM, \CC, w \Vdash}
\newcommand{\mcwf}{\MM, \CC, w \not \Vdash}
\newcommand{\mcut}{\MM, \CC, u \Vdash}
\newcommand{\mcuf}{\MM, \CC, u \not \Vdash}
\newcommand{\mcaut}{\MM, \updf{\CC}{\alpha}, u \Vdash}
\newcommand{\mcauf}{\MM, \updf{\CC}{\alpha}, u \not \Vdash}
\newcommand{\mcavt}{\MM, \updf{\CC}{\alpha}, v \Vdash}
\newcommand{\mcavf}{\MM, \updf{\CC}{\alpha}, v \not \Vdash}

\newcommand{\Cotn}{(\defg{\alpha_1}, \dots, \defg{\alpha_n})}
\newcommand{\Caotn}{(\defg{\alpha}, \defg{\alpha_1}, \dots, \defg{\alpha_n})}
\newcommand{\STaotn}{\sete{(\defg{\alpha}, \defg{\alpha_1}, \dots, \defg{\alpha_n})}}
\newcommand{\SQotn}{(\alpha_1, \dots, \alpha_n)}
\newcommand{\SQaotn}{(\alpha, \alpha_1, \dots, \alpha_n)}

\newcommand{\eque}{= \emptyset}
\newcommand{\note}{\neq \emptyset}

%%%%%%%%%%%%%%%%%%%%%%%%%%%%%%%%%%%%%%%%
\subsection{Proofs about expressivity of $\conwon$}
\label{subsec:Proofs about expressivity of ConWON}
%%%%%%%%%%%%%%%%%%%%%%%%%%%%%%%%%%%%%%%%

%%%%%%%%%%%%%%%%%%%%%%%%%%%%%%%%%%%%%%%%%%%%%
%%%%%%%%%%%%%%%%%%%%%%%%%%%%%%%%%%%%%%%%%%%%%
\begin{lemma} \label{lemma:basic}
Fix a model $\MM$.

\begin{enumerate}
\item $\sete{\updf{\CC}{\alpha}} = \emptyset$ if and only if $\defg{\alpha} = \emptyset$. \label{item:basic emptyset}
\item $\sete{\updf{(\updf{\CC}{\alpha})}{\beta}} = \sete{\updf{\CC}{(\alpha \land \beta)}}$, where $\defg{\alpha \land \beta} \neq \emptyset$. \label{item:basic alpha plus beta alpha and beta}
\item $\sete{\updf{(\updf{\CC}{\alpha})}{\beta}} = \sete{\updf{\theta}{\beta}}$, where $\defg{\alpha \land \beta} = \emptyset$. \label{item:basic alpha plus beta beta}
\item $\sete{\updf{\CC}{\alpha}} \subseteq \defg{\alpha}$. \label{item:basic alpha alpha}
\end{enumerate}
\end{lemma}

This result is easy to show.

%%%%%%%%%%%%%%%%%%%%%%%%%%%%%%%%%%%%%%%%%%%%%
%%%%%%%%%%%%%%%%%%%%%%%%%%%%%%%%%%%%%%%%%%%%%
\begin{lemma} \label{lemma:wuyu}
Let $\chi$ be a closed formula. Assume $\sete{\updf{\CC}{\alpha}} \neq \emptyset$. Then for all $w$ and $u$ of $\MM$, $\MM, \CC, w \Vdash \cond{\alpha} \chi$ if and only if $\MM, \updf{\CC}{\alpha}, u \Vdash \chi$.
\end{lemma}

\begin{proof}
Assume $\MM, \CC, w \not \Vdash \cond{\alpha}\chi$. Then $\MM, \updf{\CC}{\alpha}, x \not \Vdash \chi$ for some $x \in \sete{\updf{\CC}{\alpha}}$. Then $\MM, \updf{\CC}{\alpha}, u \not \Vdash \chi$. Assume $\MM, \updf{\CC}{\alpha}, u \not \Vdash \chi$. Let $x \in \sete{\updf{\CC}{\alpha}}$. Then $\MM, \updf{\CC}{\alpha}, x \not \Vdash \chi$. Then $\MM, \CC, w \not \Vdash \cond{\alpha}\chi$.
\end{proof}

\setcounter{lemma}{1}

%%%%%%%%%%%%%%%%%%%%%%%%%%%%%%%%%%%%%%%%%%%%%
%%%%%%%%%%%%%%%%%%%%%%%%%%%%%%%%%%%%%%%%%%%%%
\begin{lemma}
The following formulas are valid, where $\alpha, \beta$ and $\gamma$ are in $\Phi_\PL$:
\begin{enumerate}
\item $\cond{\alpha} (\phi \land \psi) \leftrightarrow (\cond{\alpha} \phi \land \cond{\alpha} \psi)$
\item $\cond{\alpha} (\phi \lor \chi) \leftrightarrow (\cond{\alpha} \phi \lor \cond{\alpha} \chi)$, where $\chi$ is a closed formula
\item $\cond{\alpha} \cond{\beta} \gamma \leftrightarrow (\EEE\alpha \rightarrow ((\EEE(\alpha \land \beta) \land \cond{\alpha \land \beta} \gamma) \lor (\neg \EEE(\alpha \land \beta) \land \AAA(\beta \rightarrow \gamma))))$
\item $\cond{\alpha} \dcon{\beta} \gamma \leftrightarrow (\EEE\alpha \rightarrow ((\EEE(\alpha \land \beta) \land \dcon{\alpha \land \beta} \gamma) \lor (\neg \EEE(\alpha \land \beta) \land \EEE(\beta \land \gamma))))$
\end{enumerate}
\end{lemma}

\setcounter{lemma}{5}

\begin{proof}
~

1. This item is easy.

2. Assume $\mcwf \cond{\alpha} \phi \lor \cond{\alpha} \chi$. Then $\mcwf \cond{\alpha} \phi$ and $\mcwf \cond{\alpha} \chi$. Then there is $u \in \cpa$ such that $\mcauf \phi$ and there is $v \in \cpa$ such that $\mcavf \chi$. As $\chi$ is a closed formula, $\mcauf \chi$. Then $\mcauf \phi \lor \chi$. Then $\mcwf \cond{\alpha} (\phi \lor \chi)$. The other direction is easy.

3. Assume $\mcwf \EEE \alpha$. Then both sides of the equivalence hold at $(\MM,\CC,w)$ trivially.

Assume $\mcwt \EEE \alpha$ and $\mcwt \EEE (\alpha \land \beta)$. Note $\cpa \neq \emptyset$ by item \ref{item:basic emptyset} in Lemma \ref{lemma:basic}. Also note $\sete{\updf{(\updf{\CC}{\alpha})}{\beta}} = \sete{\updf{\CC}{(\alpha \land \beta)}}$ by Item \ref{item:basic alpha plus beta alpha and beta} in Lemma \ref{lemma:basic}.

Assume $\mcwt \cond{\alpha}\cond{\beta} \gamma$. Let $u \in \cpa$. By Lemma \ref{lemma:wuyu}, $\mcaut \cond{\beta}\gamma$. Then for every $v \in \sete{\updf{(\updf{\CC}{\alpha})}{\beta}}$, $\MM,\updf{(\updf{\CC}{\alpha})}{\beta},v \Vdash \gamma$. Then for every $v \in \sete{\updf{\CC}{(\alpha \land \beta)}}$, $\MM,\updf{\CC}{(\alpha \land \beta)},v \Vdash \gamma$. Then $\mcwt \cond{\alpha \land \beta} \gamma$. Then $\mcwt \EEE\alpha \rightarrow ((\EEE(\alpha \land \beta) \land \cond{\alpha \land \beta} \gamma) \lor (\neg \EEE(\alpha \land \beta) \land \AAA(\beta \rightarrow \gamma)))$.

Assume $\mcwt \EEE\alpha \rightarrow ((\EEE(\alpha \land \beta) \land \cond{\alpha \land \beta} \gamma) \lor (\neg \EEE(\alpha \land \beta) \land \AAA(\beta \rightarrow \gamma)))$. Then $\mcwt \cond{\alpha \land \beta} \gamma$. Then for every $u \in \sete{\updf{\CC}{(\alpha \land \beta)}}$, $\MM,\updf{\CC}{(\alpha \land \beta)},u \Vdash \gamma$. Then for every $u \in \sete{\updf{(\updf{\CC}{\alpha})}{\beta}}$, $\MM,\updf{(\updf{\CC}{\alpha})}{\beta},u \Vdash \gamma$. Let $v \in \sete{\updf{\CC}{\alpha}}$. Then $\MM,\updf{\CC}{\alpha},v \Vdash \cond{\beta}\gamma$. By Lemma \ref{lemma:wuyu}, $\MM,\CC,w \Vdash \cond{\alpha}\cond{\beta}\gamma$.

Assume $\mcwt \EEE \alpha$ and $\mcwf \EEE (\alpha \land \beta)$. Note $\cpa \neq \emptyset$ by item \ref{item:basic emptyset} in Lemma \ref{lemma:basic}. Also note $\sete{\updf{(\updf{\CC}{\alpha})}{\beta}} = \sete{\updf{\theta}{\beta}}$ by Item \ref{item:basic alpha plus beta beta} in Lemma \ref{lemma:basic}.

Assume $\mcwt \cond{\alpha}\cond{\beta} \gamma$. Let $u \in \cpa$. By Lemma \ref{lemma:wuyu}, $\mcaut \cond{\beta}\gamma$. Then for every $v \in \sete{\updf{(\updf{\CC}{\alpha})}{\beta}}$, $\MM,\updf{(\updf{\CC}{\alpha})}{\beta},v \Vdash \gamma$. Then for every $v \in \sete{\updf{\theta}{\beta}}$, $\MM,\updf{\theta}{\beta},v \Vdash \gamma$. Then $\MM,\CC,w \Vdash \AAA(\beta \rightarrow \gamma)$. Then $\mcwt \EEE\alpha \rightarrow ((\EEE(\alpha \land \beta) \land \cond{\alpha \land \beta} \gamma) \lor (\neg \EEE(\alpha \land \beta) \land \AAA(\beta \rightarrow \gamma)))$.

Assume $\mcwt \EEE\alpha \rightarrow ((\EEE(\alpha \land \beta) \land \cond{\alpha \land \beta} \gamma) \lor (\neg \EEE(\alpha \land \beta) \land \AAA(\beta \rightarrow \gamma)))$. Then $\MM,\CC,w \Vdash \AAA(\beta \rightarrow \gamma)$. Then for every $u \in \sete{\updf{\theta}{\beta}}$, $\MM,\updf{\theta}{\beta},u \Vdash \gamma$. Then for every $u \in \sete{\updf{(\updf{\CC}{\alpha})}{\beta}}$, $\MM,\updf{(\updf{\CC}{\alpha})}{\beta},u \Vdash \gamma$. Let $v \in \sete{\updf{\CC}{\alpha}}$. Then $\MM,\updf{\CC}{\alpha},v \Vdash \cond{\beta}\gamma$. By Lemma \ref{lemma:wuyu}, $\mcwt \cond{\alpha}\cond{\beta} \gamma$.

4. Assume $\mcwf \EEE \alpha$. Then both sides of the equivalence hold at $(\MM,\CC,w)$ trivially.

Assume $\mcwt \EEE \alpha$ and $\mcwt \EEE (\alpha \land \beta)$. Note $\cpa \neq \emptyset$ by item \ref{item:basic emptyset} in Lemma \ref{lemma:basic}. Also note $\sete{\updf{(\updf{\CC}{\alpha})}{\beta}} = \sete{\updf{\CC}{(\alpha \land \beta)}}$ by Item \ref{item:basic alpha plus beta alpha and beta} in Lemma \ref{lemma:basic}.

Assume $\mcwt \cond{\alpha}\dcon{\beta} \gamma$. Let $u \in \cpa$. By Lemma \ref{lemma:wuyu}, $\mcaut \dcon{\beta}\gamma$. Then there is $v \in \sete{\updf{(\updf{\CC}{\alpha})}{\beta}}$ such that $\MM,\updf{(\updf{\CC}{\alpha})}{\beta},v \Vdash \gamma$. Then $v \in \sete{\updf{\CC}{(\alpha \land \beta)}}$ and $\MM,\updf{\CC}{(\alpha \land \beta)},v \Vdash \gamma$. Then $\mcwt \dcon{\alpha \land \beta} \gamma$. Then $\MM,\CC,w \Vdash \EEE\alpha \rightarrow ((\EEE(\alpha \land \beta) \land \dcon{\alpha \land \beta} \gamma) \lor (\neg \EEE(\alpha \land \beta) \land \EEE(\beta \land \gamma)))$.

Assume $\MM,\CC,w \Vdash \EEE\alpha \rightarrow ((\EEE(\alpha \land \beta) \land \dcon{\alpha \land \beta} \gamma) \lor (\neg \EEE(\alpha \land \beta) \land \EEE(\beta \land \gamma)))$. Then $\mcwt \dcon{\alpha \land \beta} \gamma$. Then there is $u \in \sete{\updf{\CC}{(\alpha \land \beta)}}$ such that $\MM,\updf{\CC}{(\alpha \land \beta)},u \Vdash \gamma$. Then $u \in \sete{\updf{(\updf{\CC}{\alpha})}{\beta}}$ and $\MM,\updf{(\updf{\CC}{\alpha})}{\beta},u \Vdash \gamma$. Let $v \in \cpa$. Then $\MM,\updf{\CC}{\alpha},v \Vdash \dcon{\beta}\gamma$. By Lemma \ref{lemma:wuyu}, $\mcwt \cond{\alpha}\dcon{\beta} \gamma$.

Assume $\mcwt \EEE \alpha$ and $\mcwf \EEE (\alpha \land \beta)$. Note $\cpa \neq \emptyset$ by item \ref{item:basic emptyset} in Lemma \ref{lemma:basic}. Also note $\sete{\updf{(\updf{\CC}{\alpha})}{\beta}} = \sete{\updf{\theta}{\beta}}$ by Item \ref{item:basic alpha plus beta beta} in Lemma \ref{lemma:basic}.

Assume $\mcwt \cond{\alpha}\dcon{\beta} \gamma$. Let $u \in \cpa$. By Lemma \ref{lemma:wuyu}, $\mcaut \dcon{\beta}\gamma$. Then there is $v \in \sete{\updf{(\updf{\CC}{\alpha})}{\beta}}$ such that $\MM,\updf{(\updf{\CC}{\alpha})}{\beta},v \Vdash \gamma$. Then $v \in \sete{\updf{\theta}{\beta}}$ and $\MM,\updf{\theta}{\beta},v \Vdash \gamma$. Then $\MM,\CC,w \Vdash \EEE (\beta \land \gamma)$. Then $\MM,\CC,w \Vdash \EEE\alpha \rightarrow ((\EEE(\alpha \land \beta) \land \dcon{\alpha \land \beta} \gamma) \lor (\neg \EEE(\alpha \land \beta) \land \EEE(\beta \land \gamma)))$.

Assume $\MM,\CC,w \Vdash \EEE\alpha \rightarrow ((\EEE(\alpha \land \beta) \land \dcon{\alpha \land \beta} \gamma) \lor (\neg \EEE(\alpha \land \beta) \land \EEE(\beta \land \gamma)))$. Then $\MM,\CC,w \Vdash \EEE(\beta \land \gamma)$. Then there is $u \in \sete{\updf{\theta}{\beta}}$ such that $\MM,\updf{\theta}{\beta},u \Vdash \gamma$. Then $u \in \sete{\updf{(\updf{\CC}{\alpha})}{\beta}}$ and $\MM,\updf{(\updf{\CC}{\alpha})}{\beta},u \Vdash \gamma$. Let $v \in \sete{\updf{\CC}{\alpha}}$. Then $\MM,\updf{\CC}{\alpha},v \Vdash \dcon{\beta}\gamma$. By Lemma \ref{lemma:wuyu}, $\mcwt \cond{\alpha}\dcon{\beta} \gamma$.
\end{proof}

\setcounter{theorem}{2}

%%%%%%%%%%%%%%%%%%%%%%%%%%%%%%%%%%%%%%%%%%%%%
%%%%%%%%%%%%%%%%%%%%%%%%%%%%%%%%%%%%%%%%%%%%%
\begin{theorem}
There is an effective function $\sigma$ from $\Phi_{\conwon}$ to $\Phi_{\conwon\text{-}\mathsf{1}}$ such that for every $\phi \in \Phi_{\conwon}$, $\phi \leftrightarrow \sigma(\phi)$ is valid.
\end{theorem}

\begin{proof}
~

We define the modal depth of formulas of $\Phi_{\conwon}$ with respect to $\cond{\cdot}$ in the usual way.

Pick a formula $\phi$ in $\Phi_{\conwon}$. Repeat the following steps until we cannot proceed.

\begin{itemize}
\item Pick a sub-formula $\cond{\alpha} \psi$ of $\phi$ whose modal depth with respect to $\cond{\cdot}$ is $2$ if $\phi$ has such a sub-formula.
\item Transform $\psi$ to $\chi_1 \land \dots \land \chi_n$, where all $\chi_i$ is in the form of $(\beta_1 \lor \dots \lor \beta_k) \lor (\cond{\gamma_1} \lambda_1 \lor \dots \lor \cond{\gamma_l} \lambda_l) \lor (\dcon{\eta_1} \theta_1 \lor \dots \lor \dcon{\eta_m} \theta_m)$, where all $\beta_i, \gamma_i, \lambda_i, \eta_i$ and $\theta_i$ are in $\Phi_\PL$.
\item Note $\cond{\alpha} \psi \leftrightarrow (\cond{\alpha}\chi_1 \land \dots \land \cond{\alpha}\chi_n)$ is valid by Item \ref{validity:conditionals distribution conjunction} in Lemma \ref{lemma:partial reduction conditionals valid}. Repeat the following steps until we cannot proceed:
\begin{itemize}
\item From $\cond{\alpha}\chi_1 \land \dots \land \cond{\alpha}\chi_n$, pick a conjunct $\cond{\alpha} \chi_i = \cond{\alpha} ((\beta_1 \lor \dots \lor \beta_k) \lor (\cond{\gamma_1} \lambda_1 \lor \dots \lor \cond{\gamma_l} \lambda_l) \lor (\dcon{\eta_1} \theta_1 \lor \dots \lor \dcon{\eta_m} \theta_m))$.
\item By Item \ref{validity:conditionals distribution conditionally distribution} in Lemma \ref{lemma:partial reduction conditionals valid}, $\cond{\alpha} \chi_i \leftrightarrow \xi$ is valid in $\conwon$, where $\xi = \cond{\alpha}(\beta_1 \lor \dots \lor \beta_k) \lor (\cond{\alpha}\cond{\gamma_1} \lambda_1 \lor \dots \lor \cond{\alpha}\cond{\gamma_l} \lambda_l) \lor (\cond{\alpha}\dcon{\eta_1} \theta_1 \lor \dots \lor \cond{\alpha}\dcon{\eta_m} \theta_m)$. In the ways specified by Items \ref{validity:conditionals conditionals} and \ref{validity:conditionals dual conditionals} in Lemma \ref{lemma:partial reduction conditionals valid}, transform $\xi$ to $\xi'$, whose modal depth with respect to $\cond{\cdot}$ is $1$.
\item Replace $\cond{\alpha} \chi_i$ by $\xi'$ in $\cond{\alpha}\chi_1 \land \dots \land \cond{\alpha}\chi_n$.
\end{itemize}
\end{itemize}

Define $\sigma(\phi)$ as the result. It is easy to see that $\sigma(\phi)$ is in $\Phi_{\conwon\text{-}\mathsf{1}}$ and $\phi \leftrightarrow \sigma(\phi)$ is valid.
\end{proof}

%%%%%%%%%%%%%%%%%%%%%%%%%%%%%%%%%%%%%%%%
\subsection{Proofs about comparisons to the conditional logic $\logv$}
\label{subsec:Proofs about comparisons to the conditional logic V}
%%%%%%%%%%%%%%%%%%%%%%%%%%%%%%%%%%%%%%%%

\newcommand{\SXztk}{X_0, \dots, X_k}
\newcommand{\SYztk}{Y_0, \dots, Y_k}
\newcommand{\Cztk}{X_0 \cap \dots \cap X_k}
\newcommand{\Cztl}{X_0 \cap \dots \cap X_l}
\newcommand{\Cotl}{X_1 \cap \dots \cap X_l}
\newcommand{\Czzl}{Z \cap X_0 \cap \dots \cap X_l}

%%%%%%%%%%%%%%%%%%%%%%%%%%%%%%%%%%%%%%%%%%%%%%%%%%%%
\subsubsection{Finite universal relational models for $\Phi_\logv$}
\label{subsubsec:Finite universal relational models for phi v}
%%%%%%%%%%%%%%%%%%%%%%%%%%%%%%%%%%%%%%%%%%%%%%%%%%%%

Note that $\Phi_{\logv\text{-}\mathsf{1}}$ contains no nested conditionals. This implies that we can just consider universal relational models without changing the set of valid formulas in $\Phi_{\logv\text{-}\mathsf{1}}$.

%%%%%%%%%%%%%%%%%%%%%%%%%%%%%%%%%%%%%%%%%%%%%%%%%%%%
%%%%%%%%%%%%%%%%%%%%%%%%%%%%%%%%%%%%%%%%%%%%%%%%%%%%
\begin{definition}[Universal relational models for $\Phi_\logv$] \label{def:relational models for Phi V-1}
A tuple $M = (W, <, V)$ is a universal relational model for $\Phi_\logv$ if

\begin{itemize}
\item $W$ and $V$ are as usual;
\item $<$ is a well-founded, irreflexive, transitive, almost connected binary relation on $W$.
\end{itemize}
\end{definition}

%%%%%%%%%%%%%%%%%%%%%%%%%%%%%%%%%%%%%%%%%%%%%%%%%%%%
%%%%%%%%%%%%%%%%%%%%%%%%%%%%%%%%%%%%%%%%%%%%%%%%%%%%
\begin{lemma} \label{lemma:sphere semantics is equivalent to relational semantics}
The class of universal relational models and the class of relational models determine the same set of valid formulas in $\Phi_{\logv\text{-}\mathsf{1}}$.
\end{lemma}

%%%%%%%%%%%%%%%%%%%%%%%%%%%%%%%%%%%%%%%%%%%%%%%%%%%%
\subsubsection{Rephrasing the semantics for $\Phi_\logv$}
\label{subsubsec:Rephrasing the semantics for phi v}
%%%%%%%%%%%%%%%%%%%%%%%%%%%%%%%%%%%%%%%%%%%%%%%%%%%%

\begin{fact}
Let $M = (W, <, V)$ be a universal relational model for $\Phi_\logv$. Define a relation $\equiv$ on $W$ as follows: For all $w$ and $u$, $w \equiv u$ if and only if $w \not < u$ and $u \not < w$. Then $\equiv$ is an equivalence relation. Let $\Delta_W$ be the partition of $W$ under $\equiv$. Define a relation $\ll$ on $\Delta_W$ as follows: For all $X$ and $Y$, $X \ll Y$ if and only if for all $x \in X$ and $y \in Y$, $x < y$. Then $\ll$ is a well-founded strict well-ordering on $\Delta_W$.
\end{fact}

\begin{definition}[Sphere models for $\Phi_\logv$] \label{def:Sphere models for Phi V$}
A tuple $M = (W, \Delta_W, \ll, V)$ is a sphere model for $\Phi_\logv$ if

\begin{itemize}
\item $W$ and $V$ are as usual;
\item $\Delta_W$ is a partition of $W$ and $\ll$ is a well-founded strict well-ordering on $\Delta_W$.
\end{itemize}
\end{definition}

%%%%%%%%%%%%%%%%%%%%%%%%%%%%%%%%%%%%%%%%%%%%%%%%%%%%
%%%%%%%%%%%%%%%%%%%%%%%%%%%%%%%%%%%%%%%%%%%%%%%%%%%%
\begin{definition}[Sphere semantics for $\Phi_\logv$] \label{def:Semantics for Phi V}
Let $M = (W, \Delta_W, \ll, V)$ be a sphere model.

\begin{tabular}{lll}
$M,w \Vdash \phi \cona \psi$ & $\Leftrightarrow$ & \parbox[t]{27em}{if there is $X \in \Delta_W$ such that $X \cap \stas{\phi} \neq \emptyset$, then $X' \cap \stas{\phi} \subseteq \stas{\psi}$, where $X'$ is the $\ll$-least element in $\Delta_W$ such that $X' \cap \stas{\phi} \neq \emptyset$}
\end{tabular}

\noindent where $\stas{\phi} = \{x \mid M,x \Vdash \phi\}$ and $\stas{\psi} = \{x \mid M,x \Vdash \psi\}$.
\end{definition}

It can be verified that the following result holds:

%%%%%%%%%%%%%%%%%%%%%%%%%%%%%%%%%%%%%%%%%%%%%%%%%%%%
%%%%%%%%%%%%%%%%%%%%%%%%%%%%%%%%%%%%%%%%%%%%%%%%%%%%
\begin{lemma} \label{lemma:sphere semantics is equivalent to relational semantics}
Sphere semantics is equivalent to relational semantics for $\Phi_{\logv\text{-}\mathsf{1}}$.
\end{lemma}

%%%%%%%%%%%%%%%%%%%%%%%%%%%%%%%%%%%%%%%%%%%%%%%%%%%%
\subsubsection{Rephrasing the semantics for $\Phi_\logv$ again}
\label{subsubsec:Rephrasing the semantics for phi v}
%%%%%%%%%%%%%%%%%%%%%%%%%%%%%%%%%%%%%%%%%%%%%%%%%%%%

%%%%%%%%%%%%%%%%%%%%%%%%%%%%%%%%%%%%%%%%%%%%%%%%%%%%
%%%%%%%%%%%%%%%%%%%%%%%%%%%%%%%%%%%%%%%%%%%%%%%%%%%%
\begin{definition}[Pseudo sphere models for $\Phi_\logv$] \label{def:Pseudo sphere models for Phi V}
A tuple $M = (W, \Pi, V)$ is a pseudo sphere model for $\Phi_\logv$ if

\begin{itemize}
\item $W$ and $V$ are as usual;
\item $\Pi = (X_0, \dots, X_n, \dots)$ is a sequence of pairwise disjoint (possibly empty) subsets of $W$ such that the union of them is $W$.
\end{itemize}
\end{definition}

%%%%%%%%%%%%%%%%%%%%%%%%%%%%%%%%%%%%%%%%%%%%%%%%%%%%
%%%%%%%%%%%%%%%%%%%%%%%%%%%%%%%%%%%%%%%%%%%%%%%%%%%%
\begin{definition}[Pseudo sphere semantics for $\Phi_\logv$] \label{def:Pseudo semantics for Phi V}
Let $M = (W, \Pi, V)$ is a pseudo sphere model.

\begin{tabular}{lll}
$M,w \Vdash \phi \cona \psi$ & $\Leftrightarrow$ & \parbox[t]{27em}{if there is $X$ in $\Pi$ such that $X \cap \stas{\phi} \neq \emptyset$, then $X_l \cap \stas{\phi} \subseteq \stas{\psi}$, where $l$ is the least number such that $X_l \cap \stas{\phi} \neq \emptyset$}
\end{tabular}

\noindent where $\stas{\phi} = \{x \mid M,x \Vdash \phi\}$ and $\stas{\psi} = \{x \mid M,x \Vdash \psi\}$.
\end{definition}

It can be verified that the following result holds:

%%%%%%%%%%%%%%%%%%%%%%%%%%%%%%%%%%%%%%%%%%%%%%%%%%%%
%%%%%%%%%%%%%%%%%%%%%%%%%%%%%%%%%%%%%%%%%%%%%%%%%%%%
\begin{lemma} \label{lemma:pseudo sphere semantics is equivalent to sphere semantics}
Pseudo sphere semantics is equivalent to sphere semantics for $\Phi_{\logv\text{-}\mathsf{1}}$.
\end{lemma}

By the following result, which is easy to show, empty elements in $\Pi$ do not matter in pseudo sphere semantics.

%%%%%%%%%%%%%%%%%%%%%%%%%%%%%%%%%%%%%%%%%%%%%%%%%%%%
%%%%%%%%%%%%%%%%%%%%%%%%%%%%%%%%%%%%%%%%%%%%%%%%%%%%
\begin{lemma} \label{lemma:removing empty elements}
Let $(M,w)$ and $(M',w)$ be two pointed pseudo sphere models, where $M = (W, \Pi, V)$ and $M' = (W, \Pi', V)$. Assume that $\Pi$ and $\Pi'$ are identical if we remove all the empty elements in them. Then $(M,w)$ and $(M',w)$ are equivalent for $\Phi_{\logv\text{-}\mathsf{1}}$.
\end{lemma}

%%%%%%%%%%%%%%%%%%%%%%%%%%%%%%%%%%%%%%%%%%%%%%%%%%%%
\subsubsection{Two transformation lemmas}
\label{subsubsec:Two transformation lemmas}
%%%%%%%%%%%%%%%%%%%%%%%%%%%%%%%%%%%%%%%%%%%%%%%%%%%%

%%%%%%%%%%%%%%%%%%%%%%%%%%%%%%%%%%%%%%%%%%%%%%%%%
%%%%%%%%%%%%%%%%%%%%%%%%%%%%%%%%%%%%%%%%%%%%%%%%%
\begin{lemma} \label{lemma:from acl models to conwon models}
Let $W$ be a nonempty set of states. Let $\SYztk$ be a sequence of pairwise disjoint nonempty subsets of $W$ such that $Y_0 \cup \dots \cup Y_k = W$. Define a sequence $\SXztk$ as follows:

\begin{itemize}
\item $X_0 = Y_0 \cup \dots \cup Y_k$
\item $X_1 = Y_1 \cup \dots \cup Y_k$
\item[$\vdots$]
\item $X_k = Y_k$
\end{itemize}

\begin{enumerate}
\item Let $Z \subseteq W$. Let $l$ be the greatest number such that $l \leq k$ and $\Czzl \note$. Then $\Czzl = Z \cap Y_l$ and $l$ is the greatest number such that $Z \cap Y_l \note$.
\item Let $Z \subseteq W$. Let $l$ be the greatest number such that $l \leq k$ and $Z \cap Y_l \note$. Then $\Czzl = Z \cap Y_l$ and $l$ is the greatest number such that $\Czzl \note$.
\end{enumerate}
\end{lemma}

\begin{proof}
~

1. Assume $l = k$. Note $\Cztk = X_k = Y_k$. It is easy to see that the result holds.

Assume $l < k$.

We first show that $l$ is the greatest number such that $Z \cap Y_l \note$.

Note $\Cztl = X_l$. Then $Z \cap X_l \note$. Note $Z \cap X_0 \cap \dots \cap X_{l+1} \eque$ and $X_0 \cap \dots \cap X_{l+1} = X_{l+1}$. Then $Z \cap X_{l+1} = \emptyset$. Note $X_l = Y_l \cup \dots \cup Y_k$ and $X_{l+1} = Y_{l+1} \cup \dots \cup Y_k$. Then $Z \cap Y_{l+1} \eque, \dots, Z \cap Y_k \eque$. Then $Z \cap Y_l \note$. Then $l$ is the greatest number such that $Z \cap Y_l \note$.

Then we show $\Czzl = Z \cap Y_l$.

Let $a \in \Czzl$. Then $a \in Z \cap X_l$. Note $X_l = Y_l \cup \dots \cup Y_k$. Then $a \in Z \cap (Y_l \cup \dots \cup Y_k)$. We claim $a \notin Y_{l+1}, \dots, a \notin Y_k$. Why? Suppose $a \in Y_{l+1}$. Note $X_{l+1} = Y_{l+1} \cup \dots \cup Y_k$. Then $a \in X_{l+1}$. Then $a \in Z \cap X_0 \cap \dots \cap X_{l+1}$. Then $l$ is not the greatest number such that $\Czzl \note$. We have a contradiction. Similarly, we know $a \notin Y_{l+2}, \dots, a \notin Y_k$. Then $a \in Y_l$. Then $a \in Z \cap Y_l$.

Let $a \in Z \cap Y_l$. By the definitions of $X_0, \dots, X_l$, we know $a \in X_0, \dots, a \in X_l$. Then $a \in Z \cap \Cztl$.

2. Note $\Cztl = X_l = Y_l \cup \dots \cup Y_k$ and $X_0 \cap \dots \cap X_{l+1} = X_{l+1} = Y_{l+1} \cup \dots \cup Y_k$. Also note $Z \cap Y_{l+1} = \emptyset$, \dots, $Z \cap Y_k = \emptyset$.

We first show that $l$ is the greatest number such that $\Czzl \note$.

Note $l$ is the greatest number such that $Z \cap Y_l \note$. Then $Z \cap \Cztl \note$. Then $Z \cap X_{l+1} \eque$. Then $Z \cap \Cztl \cap X_{l+1} = \emptyset$. Then $l$ is the greatest number such that $\Czzl \note$.

Then we show $\Czzl = Z \cap Y_l$.

Let $a \in Z \cap \Cztl$. Then $a \in Z \cap X_l$. Then $a \in Z \cap (Y_l \cup \dots \cup Y_k)$. Note $l$ is the greatest number such that $Z \cap Y_l \note$. Then $a \in Z \cap Y_l$.

Let $a \in Z \cap Y_l$. Then $a \in Y_l \cup \dots \cup Y_k = X_l = \Cztl$. Then $a \in Z \cap \Cztl$.
\end{proof}

%%%%%%%%%%%%%%%%%%%%%%%%%%%%%%%%%%%%%%%%%%%%%%%%%
%%%%%%%%%%%%%%%%%%%%%%%%%%%%%%%%%%%%%%%%%%%%%%%%%
\begin{lemma} \label{lemma:from conwon models to acl models}
Let $W$ be a nonempty set of states. Let $\SXztk$ be a sequence of subsets of $W$, where $X_0 = W$. Define a sequence $\SYztk$ as follows:

\begin{itemize}
\item $Y_0 = X_0-X_1$
\item $Y_1 = X_0 \cap X_1 - X_2$ 
\item[$\vdots$]
\item $Y_k = \Cztk$
\end{itemize}

\begin{enumerate}
\item Then $\SYztk$ are pairwise disjoint and $Y_0 \cup \dots \cup Y_k = W$.
\item Let $Z \subseteq W$. Let $l$ be the greatest number such that $l \leq k$ and $\Czzl \note$. Then $\Czzl = Z \cap Y_l$ and $l$ is the greatest number such that $Z \cap Y_l \note$.
\item Let $Z \subseteq W$. Let $l$ be the greatest number such that $l \leq k$ and $Z \cap Y_l \note$. Then $\Czzl = Z \cap Y_l$ and $l$ is the greatest number such that $\Czzl \note$.
\end{enumerate}
\end{lemma}

\begin{proof}
~

1. Let $i, j \leq k$ be such that $i \neq j$. We want to show $Y_i \cap Y_j = \emptyset$. Without loss of any generality, assume $i < j$. Assume $Y_i \cap Y_j \neq \emptyset$. Then there is $a$ such that $a \in Y_i$ and $a \in Y_j$. Note $Y_i = X_0 \cap \dots \cap X_i - X_{i+1}$ and $Y_j = X_0 \cap \dots \cap X_j - X_{j+1}$. Then $a \notin X_{i+1}$ and $a \in X_{i+1}$. There is a contradiction.

Let $a \in W$. Let $l$ be the the greatest number such that $a \in \Cztl$. Note $l$ exists, as $X_0 = W$. Suppose $l = k$. Then $a \in Y_k$. Suppose $l < k$. Then $a \notin X_{l+1}$. Then by the definition of $Y_l$, $a \in Y_l$.

2. We first show $\Czzl = Z \cap Y_l$.

Let $a \in \Czzl$. As $l$ is the greatest number such that $\Czzl \note$, $a \notin X_{l+1}$. Note $Y_l = \Cztk - X_{l+1}$. Then $a \in Y_l$. Then $a \in Z \cap Y_l$.

Let $a \in Z \cap Y_l$. By the definition of $Y_l$, $a \in \Cztl$. Then $a \in \Czzl$.

Then we show that $l$ is the greatest number such that $Z \cap Y_l \note$.

Assume that there is a natural number $l' > l$ such that $l' \leq k$ and $Z \cap Y_{l'} \note$. Let $a \in Z \cap Y_{l'}$. By the definition of $Y_{l'}$, $a \in Z \cap X_0 \cap \dots \cap X_{l'}$. Then $l$ is not the greatest number such that $\Czzl \note$. There is a contradiction.

3. We first show that $l$ is the greatest number such that $\Czzl \note$.

Assume that there is a natural number $l' > l$ such that $l' \leq k$ and $l'$ is the greatest number such that $Z \cap X_0 \cap \dots \cap X_{l'} \note$. Assume $l' = k$. Note $Y_k = \Cztk$. Then $Z \cap Y_k \note$. Then $l$ is not the greatest number such that $Z \cap Y_l \note$. We have a contradiction. Assume $l' < k$. Let $a \in Z \cap X_0 \cap \dots \cap X_{l'}$. As $l'$ is the greatest number such that $Z \cap X_0 \cap \dots \cap X_{l'} \note$, we know $a \notin X_{l' + 1}$. By the definition of $Y_{l'}$, $a \in Z \cap Y_{l'}$. Then $l$ is not the greatest number such that $Z \cap Y_l \note$. We have a contradiction.

Then we show $\Czzl = Z \cap Y_l$.

Let $a \in \Czzl$. As $l$ is the greatest number such that $\Czzl \note$, $a \notin X_{l+1}$. By the definition of $Y_{l}$, $a \in Y_l$. Then $a \in Z \cap Y_l$.

Let $a \in Z \cap Y_l$. By the definition of $Y_{l}$, $a \in \Cztl$. Then $a \in \Czzl$.
\end{proof}

%%%%%%%%%%%%%%%%%%%%%%%%%%%%%%%%%%%%%%%%%%%%%%%%%%%%
\subsubsection{The equivalence theorem}
\label{subsubsec:The equivalence theorem}
%%%%%%%%%%%%%%%%%%%%%%%%%%%%%%%%%%%%%%%%%%%%%%%%%%%%

\setcounter{theorem}{7}

%%%%%%%%%%%%%%%%%%%%%%%%%%%%%%%%%%%%%%%%%%%%%%%%%
%%%%%%%%%%%%%%%%%%%%%%%%%%%%%%%%%%%%%%%%%%%%%%%%%
\begin{theorem}
For every $\phi \in \Phi_{\logv\text{-}\mathsf{1}}$, $\phi$ is valid in $\logv$ if and only if $\phi$ is valid in $\conwon$.
\end{theorem}

\begin{proof}
Let $\alpha$ and $\beta$ be in $\Phi_\PL$. It suffices to show that $\alpha \cona \beta$ is satisfiable in $\logv$ if and only if $\cond{\alpha} \beta$ is satisfiable in $\conwon$.

Assume $\alpha \cona \beta$ is satisfiable in $\logv$. Then $\alpha \cona \beta$ is true at a pointed sphere model for $\Phi_\logv$. As $\Phi_\logv$ has the finite model property, $\alpha \cona \beta$ is true at a pointed finite sphere model $(M,w)$ for $\Phi_\logv$. Let $M = (W,\Pi,V)$, where $\Pi = (Y_k, \dots, Y_0)$.

Define $\MM' = (W, V)$, which is a model for $\conwon$. From the sequence $(Y_0, \dots, K_k)$, define a context $\CC = (\SXztk)$ for $\MM'$ as follows:

\begin{itemize}
\item $X_0 = Y_0 \cup \dots \cup Y_k$
\item $X_1 = Y_1 \cup \dots \cup Y_k$
\item[$\vdots$]
\item $X_k = Y_k$
\end{itemize}

Assume $\defg{\alpha} \eque$. Let $x \in W$. Then $\MM',W,x \Vdash \cond{\alpha} \beta$. Here $W$ indicates a special context. Then $\cond{\alpha} \beta$ is satisfiable in $\conwon$.

Assume $\defg{\alpha} \note$. As $Y_0 \cup \dots \cup Y_k = W$, there is $Y_i$ such that $\defg{\alpha} \cap Y_i \note$. Let $l$ be the greatest number such that $l \leq k$ and $\defg{\alpha} \cap Y_l \note$. Note $M,w \Vdash \alpha \cona \beta$. Then $\defg{\alpha} \cap Y_l \subseteq \defg{\beta}$. By Lemma \ref{lemma:from acl models to conwon models}, $\defg{\alpha} \cap \Cztl = \defg{\alpha} \cap Y_l$ and $l$ is the greatest number such that $\defg{\alpha} \cap \Cotl \note$. Then $\sete{\updf{\CC}{\alpha}} = \defg{\alpha} \cap \Cotl$. Then $\sete{\updf{\CC}{\alpha}} \subseteq \defg{\beta}$. Let $x \in \sete{\updf{\CC}{\alpha}}$. Then $\MM',\CC,x \Vdash \cond{\alpha} \beta$. Then $\cond{\alpha} \beta$ is satisfiable in $\conwon$.

Assume $\cond{\alpha} \beta$ is satisfiable in $\conwon$. Then $\cond{\alpha} \beta$ is true at a contextualized pointed model $(\MM,\CC,w)$ for $\conwon$. Then $\sete{\updf{\CC}{\alpha}} \subseteq \defg{\beta}$. Let $\MM = (W,V)$. Define $\CC' = \updd{\CC}{W}$. Let $\CC' = (\SXztk)$. Note $X_0 = W$. It can be verified that $\sete{\updf{\CC}{\alpha}} = \sete{\updf{\CC'}{\alpha}}$. Then $\sete{\updf{\CC}{\alpha'}} \subseteq \defg{\beta}$.

Define a sequence $(\SYztk)$ as follows:

\begin{itemize}
\item $Y_0 = X_0-X_1$
\item $Y_1 = X_0 \cap X_1 - X_2$ 
\item[$\vdots$]
\item $Y_k = \Cztk$
\end{itemize}

\noindent By Lemma \ref{lemma:from conwon models to acl models}, $\SYztk$ are pairwise disjoint and $Y_0 \cup \dots \cup Y_k = W$. Let $\Pi = (Y_k, \dots, Y_0)$. Define $M' = (W,\Pi,V)$, which is a pseudo sphere model for $\logv$.

Assume $\defg{\alpha} \eque$. Then there is no $Y_i$ in $\Pi$ such that $\defg{\alpha} \cap Y_i \neq \emptyset$. Let $x \in W$. Then $M',x \Vdash \alpha \cona \beta$. Then $\alpha \cona \beta$ is satisfiable in $\logv$.

Assume $\defg{\alpha} \note$. Then $\sete{\updf{\CC'}{\alpha}} \note$. Let $\sete{\updf{\CC'}{\alpha}} = \defg{\alpha} \cap \Cztl$. By Lemma \ref{lemma:from conwon models to acl models}, $\defg{\alpha} \cap \Cztl = \defg{\alpha} \cap Y_l$ and $l$ is the greatest number such that $l \leq k$ and $\defg{\alpha} \cap Y_l \note$. Note $\sete{\updf{\CC}{\alpha'}} \subseteq \defg{\beta}$. Then $\defg{\alpha} \cap Y_l \subseteq \defg{\beta}$. Let $x \in W$. Then $M',x \Vdash \alpha \cona \beta$. Then $\alpha \cona \beta$ is satisfiable in $\logv$.
\end{proof}

\end{document}